\def\Vol{16}
\def\Num{3}

\def\Year{2000}
\documentclass[twoside]{article}

\usepackage{coin}
\usepackage{graphicx}
\usepackage{latexsym}

\newcommand{\nm}{\mid\joinrel\sim}
\newcommand{\N}{I\kern-4ptN}
\newcommand{\qto}{\mathrel{?\kern-4pt\to}}
\newcommand{\xor}{\mathrel{\vee\kern-0.73em-}}

\title{On resolving conflicts between arguments}

\begin{document}
\setcounter{page}{469}
\maketitle
\author{Nico Roos}

\affil{Institute for Knowledge and Agent Technology, Maastricht
University,
\\ P.O. Box 616, 6200 MD Maastricht, The Netherlands, e-mail:
roos@cs.unimaas.nl}

\begin{abstract}
Argument systems are based on the idea that one can construct 
arguments for propositions; i.e., structured reasons justifying 
the belief in a proposition. Using defeasible rules, arguments 
need not be valid in all circumstances, therefore, it might be 
possible to construct an argument for a proposition as well as its 
negation. When arguments support conflicting propositions, one of
the arguments must be defeated, which raises the question of
\emph{which (sub-)arguments can be subject to defeat}?

In legal argumentation, meta-rules determine the valid arguments by
considering the last defeasible rule of each argument involved in
a conflict. Since it is easier to evaluate arguments using their
last rules, \emph{can a conflict be resolved by considering only
the last defeasible rules of the arguments involved}?

We propose a new argument system where, instead of deriving a
defeat relation between arguments, \emph{undercutting-arguments} 
for the defeat of defeasible rules are constructed. This system 
allows us, 
(\textit{i}) to resolve conflicts (a generalization of rebutting arguments) 
using only the last rules of the arguments for inconsistencies, 
(\textit{ii}) to determine a set of valid (undefeated) arguments in linear time
using an algorithm based on a JTMS, 
(\textit{iii}) to establish a relation with Default Logic, and 
(\textit{iv}) to prove closure
properties such as \emph{cumulativity}. We also propose an
extension of the argument system that enables \emph{reasoning
by cases}.
 \kw{argumentation, defeasible rules, non-monotonic reasoning.}
\end{abstract}
\footnotetext{This paper extends (Roos 1997b).}

\vspace*{8mm}

\noindent
\emph{This revised version of the paper provides a more clear 
explanation of reasoning by cases (Section \ref{RbC}), fixes an issue in the proof of the closure property Cautious Monotony mentioned in 
Theorem \ref{closure-prop}, and corrects several small errors.}

\section{Introduction}
Argument systems originate from philosophy (Toulmin 1958). More
recently they have also been studied in AI (Bondarenko et al.
1997; Cayrol 1995; Dung 1993; Dung 1995; Fox et al. 1992; Geffner
1994; Hunter 1994; Kraus et al. 1995; Lin \& Shoham 1989; Loui
1987; Loui 1998; Pollock 1987, 1992, 1994; Poole 1988; Prakken
1993; Prakken \& Vreeswijk 1999; Simari \& Loui 1992; Vreeswijk
1991, 1997). When such argument systems are used for reasoning
with defeasible rules (Fox et al. 1992; Geffner 1994; Hunter 1994;
Kraus et al. 1995; Loui 1987; Pollock 1987; Prakken 1993; Prakken
\& Vreeswijk 1999; Simari \& Loui 1992; Vreeswijk 1991, 1997), a
rule is viewed as a justification for believing the consequent of
the rule whenever we have a justification for believing its
antecedent (Toulmin 1958). A justification for believing the
antecedent can consist of facts about the world, denoted as
evidence or premises, and of propositions that are justified by
other defeasible rules. So, we can construct a tree of defeasible
rules that justifies the belief in some proposition with respect
to some evidence. This tree is called an {\em argument} for the
proposition.

Since the rules used in the construction of arguments are
defeasible, it might be possible to construct an argument for a
proposition as well as its negation. Clearly, only one of these
arguments can give a valid justification for the proposition it
supports. In most argument systems proposed in the literature, one
of the arguments supporting the conflicting propositions; i.e.\ for
a proposition and its negation, is defeated (Loui 1987; Pollock
1987, 1992, 1994; Prakken 1993; Prakken \& Vreeswijk 1999; Simari
\& Loui 1992; Vreeswijk 1991, 1997). 
Generally, if an argument is defeated, there is a defeated sub-argument 
(not necessarily a proper sub-argument) that has a single last defeasible rule,
and no sub-argument of this sub-argument is defeated.
An exception are arguments that are based on defeasible causal relations. 
Geffner (1994), for example, allows the chain of causal arguments to
be broken at any rule of an argument if the proposition supported
by the argument conflicts with an observed fact. 


Defeasible rules representing causal relations have one property,
not found in non causal defeasible rules. Causal defeasible rules
can be used in contraposition. This raises the following question.
{\em If rules cannot be used in contraposition, can a conflict be
resolved by defeating a last defeasible rule of the argument
supporting one of the conflicting propositions}?

In legal argumentation, meta rules are used to resolve conflicts
(Prakken 1993). These meta rules determine the valid arguments by
considering the last defeasible rule, with respect to the chain of
argumentation, of each proposition involved in the conflict.
Hence, an argument should not only reflect the information used in
an argumentation, but also the structure of the argumentation.

Some argument system represent this structure explicitly
(Vreeswijk 1991, 1997), while others represent it
implicitly (Pollock 1987, 1992). Vreeswijk (1991, 1997) only
considers defeasible rules that are definite horn clauses and a
special symbol $\perp$ to denote an inconsistency. In this
language, each argument for a conflict; i.e.\ for $\perp$, is
unique. If we use, however, full propositional or predicate logic,
there can be more than one way of argumentation for deriving `the
same conflict'.

Suppose for example that we have the following three defeasible
rules: $\{ a \leadsto \neg d, b \leadsto \neg e, c \leadsto (d
\vee e) \}$ and the facts: $\{ a, b, c \}$. Then we can construct
arguments for conflicting propositions in at least three different
ways.
\[ \left. \begin{array}{r} \left. \begin{array}{r} a \\
a \leadsto \neg d \end{array} \right\} \neg d \\ \left.
\begin{array}{r} c \\ c \leadsto (d \vee e) \end{array} \right\}
d \vee e \end{array} \right\} e \mbox{ and } \left.
\begin{array}{r} b \\ b \leadsto \neg e \end{array} \right\} \neg
e \]
\[ \left. \begin{array}{r} \left. \begin{array}{r} b \\
b \leadsto \neg e \end{array} \right\} \neg e \\ \left.
\begin{array}{r} c \\ c \leadsto (d \vee e) \end{array} \right\}
d \vee e \end{array} \right\} d \mbox{ and } \left.
\begin{array}{r} a \\ a \leadsto \neg d \end{array} \right\} \neg
d \]
\[ \left. \begin{array}{r} \left. \begin{array}{r} b \\
b \leadsto \neg e \end{array} \right\} \neg e \\ \left.
\begin{array}{r} a \\ a \leadsto \neg d \end{array} \right\} \neg d
\end{array} \right\} \neg d \wedge \neg e \} \neg(d \vee e) \mbox{ and }
\left. \begin{array}{r} c \\ c \leadsto d \vee e \end{array}
\right\} d \vee e \]
 Each of these three couples of arguments supporting conflicting
propositions, uses exactly the same information to derive a
conflict. We would therefore expect that we have only one argument
for the conflict instead of three.

The heart of the problem is that the consequences of the three
defeasible rules used in the above presented arguments are
logically inconsistent. How these three rules are used to derive
an inconsistency, should not matter. The derivation of the
inconsistency takes place through logically sound deductions which
cannot be subject to defeat. The line of reasoning that is
followed using the logical sound deductions should not
matter. Therefore, in the definition of an argument given in the
next section, we abstract away from the actual logical sound deductions. For
the same reason, we consider arguments for inconsistencies
$\perp$, instead of arguments of the conflicting propositions.

As mentioned above, in legal argumentation only the last rules for
an inconsistency are considered in order to resolve the
inconsistency. Using a preference relations, one of the last rules
is identified as the culprit. There are $\frac{1}{2} n (n-1)$
pairs of rules between which there can be a preference, where $n$
is the number of defeasible rules. If we use a preference on the
arguments, ignoring the structure of the arguments, there can be
at most $\frac{1}{2} 2^n (2^n-1)$ pairs of arguments between which
there can be a preference. Clearly, it is easier to evaluate
arguments using their last rules, than using the whole argument.
We will therefore investigate the following question. {\em Is
there a need for considering more than only the last rules of an
argument for an inconsistency}?

The next section formalizes the arguments that can be constructed
using defeasible rules. The here defined arguments do not only
represent the defeasible rules that are used, but also the line of
reasoning. Section 3 discusses whether we can resolve an
inconsistency just by defeating one of the last defeasible rules
of the argument for the inconsistency. Section 4 investigates
whether we can select the rule (or argument) to be defeated just
be looking at the last rules of the argument for the
inconsistency. Based on the results of Sections 3 and 4, Section 5
proposes a new argument system. What is essentially new is that
inconsistencies are resolved by constructing an argument for the
{\em undercutting} defeat of one of the defeasible rules of the
argument for the inconsistency.
Section 6 discusses how to compute an extension and presents a
linear time algorithm for doing so. Section 7 discusses closure
properties of the argument system and the relation with
default logic. Section 8 presents an extension of the presented
argument system that enables reasoning by cases. Finally, Section
9 discusses related work and Section 10 concludes the paper.

\section{The argument system} \label{argument-sys}
We will derive arguments using a defeasible theory $\langle
\Sigma, D \rangle$. Here, $\Sigma$ represents a set of premises
and $D$ represents a set of defeasible rules. The set of premises
$\Sigma$ is a subset of the propositional logic $L$. $L$ is
recursively defined from a set of atomic propositions $At$ and the
operators $\neg$, $\wedge$ and $\vee$.

For every defeasible rule $\varphi \leadsto \psi \in D$ there
holds that $\varphi$ is a proposition in $L$ and that $\psi$ is
either a proposition in $L$ or the negation of a defeasible rule
in $D$; i.e.\ $\psi = \neg ( \alpha \leadsto \beta)$ and $\alpha
\leadsto \beta \in D$. The negation of a defeasible rule $\neg (
\alpha \leadsto \beta)$ will be interpreted as: `$\alpha$ may no
longer justify $\beta$'. So the negation of a rule explicitly
blocks the conclusive force of the defeasible rule. It will be
used to describe the {\em undercutting} defeat of rule. If
we have a valid argument for $\neg ( \alpha \leadsto \beta)$, then
no argument containing the rule $\alpha \leadsto \beta$ can be
valid. For example, we can undercut the rule: `something that
looks red is red' by the rule: `something that stands below a red
light need not be red if it looks red'.

Notice the correspondence of the defeasible rules $\alpha \leadsto
\beta$ and $\varphi \leadsto \neg(\alpha \leadsto \beta)$ with
respectively the semi- and non-normal default rules $\frac{\alpha
: \beta, \omega_1}{\beta}$ and $\frac{\varphi :
\omega_2}{\neg \omega_1}$ where $\omega_1$ and $\omega_2$
summarize the exceptions on the default rules. Also notice the
difference with Nute's (1988, 1994) defeater rule $\varphi \qto
\neg \beta$. If we have a valid argument for $\varphi$, Nute's
defeater rule $\varphi \qto \neg \beta$ defeats {\em any} argument
containing a rule of the form $\alpha \leadsto \beta$. We can,
however, use the defeater rule $\varphi \wedge \alpha \qto \neg
\beta$ to describe $\varphi \leadsto \neg ( \alpha \leadsto
\beta)$.

In an argument system, a defeasible rule is viewed as a
justification for believing the consequent of the rule whenever we
have a justification for believing its antecedent (Toulmin 1958).
A justification for believing the antecedent can consist of facts
about the world, denoted as evidence or premises, and of
propositions that are justified by other defeasible rules. So, we
can construct a tree of defeasible rules that justifies the belief
in some proposition with respect to some evidence. This tree is
called an {\em argument} for the proposition.

Logically sound deduction steps need not be represented explicitly
in an argument. None of these deduction steps can be subject to
defeat. Only the relations described by defeasible rules need not
be valid in all circumstances.

\begin{definition} \label{argument}
Let $\langle \Sigma, D \rangle$ be a defeasible theory where
$\Sigma$ is the set of premises and $D$ is the set of rules.

Then an argument\footnote{We will sometimes add the index $\psi$
to an argument ($A_\psi$) to denote that it is an argument for
$\psi$. Of course there can be more than one argument for $\psi$.}
$A$ for a proposition $\psi$ is recursively defined in the
following way:
\begin{itemize}
\item
For each $\psi \in \Sigma$: $A = \{ \langle \emptyset, \psi
\rangle\}$ is an argument for $\psi$.
\item
Let $A_1, ..., A_n $ be arguments for respectively
$\varphi_1,...,\varphi_n$. If $\varphi_1,...,\varphi_n \vdash
\psi$, then $A = A_1 \cup ... \cup A_n$ is an argument for $\psi$.
\item
For each $\varphi \leadsto \psi \in D$ if $A'$ is an argument for
$\varphi$, then $A = \{ \langle  A', \varphi \leadsto \psi \rangle
\}$ is an argument for $\psi$.
\end{itemize}

Let $A = \{ \langle  A'_1, \alpha_1 \rangle ,...,
\langle A'_n, \alpha_n, \rangle \}$.
Then:
\[ \begin{array}{l}
\vec{A} = \{ \alpha_1,...,\alpha_n \} \cap D ;  \\
\hat{A} = \{ c(\alpha_1) ,..., c(\alpha_n) \} \mbox{ where
$c(\alpha) = \alpha$ if $\alpha \in L$ and,
$c(\alpha \leadsto \beta) = \beta$} ;  \\
\tilde{A} = \{ \alpha_i \mid 1 \leq i \leq n,\alpha_i \in D \} \cup
\bigcup^n_{i=1} \tilde{A}'_i ;  \\
\bar{A} = \{ \alpha_i \mid 1 \leq i \leq n,\alpha_i \in \Sigma \} \cup
\bigcup^n_{i=1} \bar{A}'_i
\end{array} \]
\end{definition}

\begin{example}
    Let $A = \{ \langle \emptyset, \alpha \rangle,
    \langle \{ \langle \{ \langle \emptyset, \beta \rangle \},
    \beta \leadsto \gamma \rangle \},
    \gamma \leadsto \delta \rangle \}$
    be an argument for $\varphi$.
    \[ \left. \begin{array}{r} \alpha  \\
    \beta \vdash \beta \leadsto \gamma \vdash \gamma \leadsto \delta
    \end{array} \right | \hspace{-5pt}- \varphi \]
    Then $\vec{A} = \{ \gamma \leadsto \delta \}$ denotes the last rules
    used in the argument $A$.
    Furthermore, $\hat{A} = \{ \alpha, \delta \}$ denotes the
    propositions that represent the beliefs $Th(\{ \alpha, \delta \})$
    supported by the argument $A$.
    Clearly, $A$ is an argument for every proposition
    $\varphi \in Th(\{ \alpha, \delta \})$.
    $\tilde{A} = \{ \gamma \leadsto \delta, \beta \leadsto \gamma \}$
    denotes the set of all rules in $A$,
    and $\bar{A} = \{ \alpha, \beta \}$ denotes the
    premises used in the argument $A$.
\end{example}

In the above definition of an argument, we do not apply the
contraposition of a defeasible rule in the construction of an
argument. In general, the contraposition of a defeasible rule is
invalid. A rule describes that its consequent should hold or
probably holds in context described by its antecedent. By no means
this implies that the antecedent does not hold if the consequent
does not hold.

If the defeasible rule is interpreted as describing a preference,
the negation of the consequent does not imply that the negation
the antecedent should hold. A rule describes what should hold in
the context described by its antecedent. The converse need not
hold. So, knowing that John may not drive a car, we may not
conclude that he does not own a driving license. It may just be
the case that we have an exceptional situation, e.g. John is
drunk, John has collected too many speeding tickets, John may not
drive a car on doctors orders, and so. Especially if most people
own a driving license, an exceptional situation need not be
unlikely.

Also if the defeasible rule is interpreted as describing a
conditional probability, $Pr(\psi \mid \varphi) >t$ does not imply
that $Pr(\neg \varphi \mid \neg \psi) >t$. In fact, if $Pr(\psi
\mid \varphi) < 1$, $Pr(\neg \varphi \mid \neg \psi)$ can have any
value in the interval $[0,1]$. Only in the event that we also know
the a priori probabilities of $Pr(\varphi)$ and of $Pr(\psi)$, we
can verify whether $Pr(\neg \varphi \mid \neg \psi) >t$ holds.

{\em Causal rules} are a special kind of defeasible rules that do
possess a contraposition (Geffner 1994). If, `{\em normally},
$\varphi$ {\bf causes} $\psi$', then $\neg \psi$ implies $\neg
\varphi$, unless we have an exceptional situation. Such a rule can
be described by a conditional probability, as is done in Bayesian
Belief Networks. This description is incomplete unless we know or
we can calculate the a priori probabilities of the antecedent and
the consequent. Bayesian Belief Networks guarantee the latter.
Here, however, we do not have this information. Therefore, to
guarantee that the contraposition is applied correctly, we need a
specialized approach. Geffner (1994) discusses the properties of
such an approach. In the remainder of this paper, however, we will
not consider causal rules.

Two arguments can be related to each other. The relation that is
of interest for us is whether one argument uses the same inference
steps as another argument. If so, the former is called a
sub-argument of the latter. Though an argument can be viewed as a
tree, a sub-argument is not exactly a sub-tree.

\begin{definition}
An argument $A$ is a sub-argument of $B$, $A \leq B$, if and only
if every $\langle A', \alpha \rangle \in A$ is a sub-structure of
the argument $B$.\footnote{Notice that we reach the base of the
recursion if $A$ is an empty set. If $A$ is an empty set, it is
trivial that every $\langle A', \alpha \rangle \in A$ is a
sub-structure of the argument $B$.}

$\langle A', \alpha \rangle$ is a sub-structure of an argument
$B$ if and only if
 \begin{itemize}
    \item   either there exists a $\langle  B', \alpha \rangle \in B$ such
    that $A'$ is a sub-argument of $B'$;

    \item  or there exists a $\langle B', \beta \rangle \in B$ such that
    $\langle  A', \alpha \rangle$ is a sub-structure of $B'$.
 \end{itemize}
\end{definition}
\begin{example}
    let $A = \{ \langle  \emptyset, \alpha \rangle,
    \langle \{ \langle \{ \langle \emptyset, \beta \rangle \},
    \beta \leadsto \gamma \rangle \},
    \gamma \leadsto \delta \rangle \}$
    be an argument for $\varphi$.
    \[ \left. \begin{array}{r} \alpha  \\
    \beta \vdash \beta \leadsto \gamma \vdash \gamma \leadsto \delta
    \end{array} \right | \hspace{-5pt}- \varphi \]
    Then
    \[ A_1 = \{ \langle  \emptyset, \alpha \rangle,
    \langle \{ \langle \{ \langle \emptyset, \beta \rangle \},
    \beta \leadsto \gamma \rangle \},
    \gamma \leadsto \delta \rangle \} \]
        \[ \left. \begin{array}{r} \alpha  \\
    \beta \vdash \beta \leadsto \gamma \vdash \gamma \leadsto \delta
    \end{array} \right | \]
    \[ A_2 = \{ \langle \emptyset, \alpha \rangle,
    \langle \{ \langle \emptyset, \beta \rangle \},
    \beta \leadsto \gamma \rangle \} \]
        \[ \left. \begin{array}{r} \alpha  \\
    \beta \vdash \beta \leadsto \gamma
    \end{array} \right | \]
    \[ A_3 = \{ \langle \emptyset, \alpha \rangle,
    \langle \emptyset, \beta \rangle \} \]
        \[ \left. \begin{array}{r} \alpha  \\
    \beta
    \end{array} \right | \]
    \[ A_4 = \{ \langle \emptyset, \alpha \rangle \} \]
        \[ \alpha \mid \]
    \[ A_5 = \{ \langle \{ \langle \{ \langle \emptyset, \beta \rangle \},
    \beta \leadsto \gamma \rangle \}, \gamma \leadsto \delta \rangle
    \} \]
        \[ \beta \vdash \beta \leadsto \gamma \vdash \gamma \leadsto
        \delta \mid \]
    \[ A_6 = \{ \langle \{ \langle \emptyset, \beta \rangle \},
    \beta \leadsto \gamma \rangle \} \]
        \[ \beta \vdash \beta \leadsto \gamma \mid \]
    \[ A_7 = \{ \langle \emptyset, \beta \rangle \} \]
        \[ \beta \mid \]
    are sub-arguments of $A$.
\end{example}

An argument represents a derivation tree of defeasible rules.
Since a rule in an argument $A$ gives a justification for its
consequent, the argument can be viewed as a global justification
for a proposition $\varphi$, $\hat{A} \vdash \varphi$, that is
grounded in the premises $\bar{A}$. Whether an argument is valid
depends on whether the argument or one of its sub-arguments is
defeated. When an argument $A$ for some proposition $\varphi$ is
valid we say that $\varphi$ follows from the premises $\bar{A}$
using the rules $\tilde{A}$.

\section{Defeating a last rule of an argument}
 \label{conflicts}
A defeasible rule $\varphi \leadsto \psi$ describes either a
preferred or a probabilistic relation. Therefore, there may exist
situations in which the relation it represents, is invalid. In
these exceptional situations, either $\neg \psi$ must holds or
both $\psi$ and $\neg \psi$ must be unknown. Since an argument is
basically a tree constructed using defeasible rules, an argument
containing a rule that is not valid in the current context, can
neither be valid.

There are two reasons for an argument to become invalid. Either
the argument contains a rule $\alpha \leadsto \beta$ while we have
a valid argument for $\neg(\alpha \leadsto \beta)$, or the
argument is a sub-argument of an argument for an inconsistency. In
the latter situation the question is, which sub-argument(s) of the
argument for an inconsistency, can no longer be valid? In the
discussion of this question, we will use the term {\em disagreeing
arguments} which is defined in the following way.
\begin{definition}
Let $A_\perp = \{ \langle A'_1, \mu_1 \rangle,...,\langle A'_n,
\mu_n \rangle \}$ be an argument for and inconsistency
($\hat{A}_\perp \vdash \perp$).

Then, the arguments $A_1 = \{ \langle  A'_1, \mu_1 \rangle \},...,
A_n = \{ \langle A'_n, \mu_n \rangle \}$ are said to {\em
disagree}.
\end{definition}

Clearly, in order to restore consistency, some of the disagreeing
arguments can no longer be valid. These arguments are said to be
defeated because of the other arguments. It is also clear that it
is sufficient to defeat only one of the disagreeing arguments in
order to restore consistency if the argument for the inconsistency
is a (subset) minimal argument. Without lost of generality, we may assume
that the argument for the inconsistency is a minimal argument.
Resolving inconsistencies using the minimal arguments will also
resolve inconsistencies based on non minimal arguments. We can
therefore reformulate the above raised question. Is it sufficient
to defeat a disagreeing argument, but no proper sub-argument of
this disagreeing argument, to resolve an inconsistency? We will
see that a set of defeasible rules can always be extended such
that indeed no proper sub-argument of a disagreeing argument needs
to be defeated.

Let \[ A_1 = \{ \langle  A'_1, \mu_1 \rangle \},..., A_n = \{
\langle A'_n, \mu_n \rangle \} \] be a set of disagreeing
arguments; i.e.\ $A =\bigcup^n_{i=1} A_i$ is an argument for
$\perp$. Suppose that some proper sub-argument $A_\varphi = \{
\langle A', \alpha \leadsto \varphi \rangle \}$ of the disagreeing
argument $A_k = \{ \langle  A'_k, \mu_k \rangle \}$ is defeated
because of the inconsistency and that no proper sub-argument of
$A_\varphi$ is defeated. Then, $\bigcup_{i=1}^n \bar{A}_i$
represents an exceptional situation in which either $\neg \varphi$
holds or $\varphi$ is unknown.

Suppose that $\neg \varphi$ holds. We cannot use the
contraposition of the rules to derive $\neg \varphi$. We can,
however, introduce rules that enable us to construct an argument
$A_{\neg \varphi}$ for $\neg \varphi$ such that $\bar{A}_{\neg
\varphi} \subseteq \bigcup_{i=1}^n \bar{A}_i$. In that case
$A_\varphi$ is a disagreeing argument in another inconsistency.
This inconsistency can be used to defeat $A_\varphi$. Hence, there
is no need for defeating $A_\varphi$ because of the argument $A$
for $\perp$. For example, let
\[ A_\perp = \{ \langle \{ \langle
\{ \langle \emptyset, \alpha \rangle \},
\alpha \leadsto \varphi \rangle \}, \varphi \leadsto \eta \rangle,
\langle \{ \langle \{ \langle \emptyset, \beta \rangle \}, \beta
\leadsto \psi \rangle \}, \psi \leadsto \neg \eta \rangle \} \]
\[ \left. \begin{array}{r} \alpha \vdash \alpha \leadsto \varphi
\vdash \varphi \leadsto \eta \\ \beta \vdash \beta \leadsto \psi
\vdash \psi \leadsto \neg \eta
\end{array} \right | \hspace{-4pt} - \perp \]
be an argument for $\perp$. Then we can defeat $\alpha \leadsto
\varphi$ by introducing the rule $\alpha \wedge \beta \leadsto
\neg \varphi$.

Now suppose that $\varphi$ is unknown. 
We cannot introduces rules that enable us to construct an argument 
$A_{\neg\varphi}$ since $\varphi$ is unknown. We can, however,
introduce rules that enable us to construct an argument $A_{\neg
(\alpha \leadsto \varphi)} = \{ \langle A', \xi \leadsto \neg
(\alpha \leadsto \varphi) \rangle \}$ for $\neg (\alpha \leadsto
\varphi)$ such that $\bar{A}_{\neg (\alpha \leadsto \varphi)}
\subseteq \bigcup_{i=1}^n \bar{A}_i$. Since $A_\varphi = \{
\langle A', \alpha \leadsto \varphi \rangle \}$ is defeated if
$A_{\neg (\alpha \leadsto \varphi)}$ is valid, there is no need
for defeating $A_\varphi$ because of the argument $A$ for $\perp$.
For example, let
\[ A_\perp = \{ \langle \{ \langle
\{ \langle \emptyset, \alpha \rangle \},
\alpha \leadsto \varphi \rangle \}, \varphi \leadsto \eta \rangle,
\langle \{ \langle \{ \langle \emptyset, \beta \rangle \}, \beta
\leadsto \psi \rangle \}, \psi \leadsto \neg \eta \rangle \} \]
\[ \left. \begin{array}{r} \alpha \vdash \alpha \leadsto \varphi
\vdash \varphi \leadsto \eta \\ \beta \vdash \beta \leadsto \psi
\vdash \psi \leadsto \neg \eta
\end{array} \right | \hspace{-4pt} - \perp \]
be an argument for $\perp$. Then we can defeat $\alpha \leadsto
\varphi$ by introducing the rule $\alpha \wedge \beta \leadsto
\neg (\alpha \leadsto \varphi)$.

Hence, we can avoid the need for defeating a proper sub-argument
of a disagreeing argument, if necessary by introducing additional
rules. Since a disagreeing argument has a unique last rule,
defeating a disagreeing argument implies defeating its {\em last
rule}. Hence, it suffice to defeat one of the last rules $\vec{A}$
of an argument $A$ for an inconsistency.

\section{A preference relation on rules}
 \label{evaluate}
In the previous section we have seen that no proper sub-argument
of one of the disagreeing arguments needs to be subject to defeat.
This makes it possible to defeat a disagreeing argument by
defeating its {\em last rule}. We will now investigate whether we
can determine the rule to be defeated by considering only the last
rule of each of the disagreeing arguments; i.e.\ the last rules of
the argument for the inconsistency.

Defeating the last rule of one of the disagreeing arguments in
case of an inconsistency offers three advantages. Firstly, we on
longer have to consider a defeat relation between arguments as is
done in: (Pollock 1987, 1994; Simari \& Loui 1992; Vreeswijk
1997). This significantly simplifies the preference relation that
we must consider. If we use a preference relation on rules, then
there are $\frac{1}{2} n (n-1)$ pairs of rules between which there
can be a preference, where $n$ is the number of defeasible rules.
If we use a preference on the arguments, ignoring the structure of
the arguments, there can be $\frac{1}{2} 2^n (2^n-1)$ pairs of
arguments between which there can be a preference.

Secondly, an argument looses its conclusive force (is defeated) if
it contains defeated rules. This simplifies the handling of
arguments. And, as we will see in Section \ref{algor}, it enables
us to determine a set of valid arguments in linear time.

Thirdly, the resolution of inconsistencies will be cumulative. It
does not matter whether the antecedent of a last rule is an
observed fact or derived through reasoning. This is an important
property since an observed fact may be based on some hidden
reasoning of which we are not aware.

To show that there is no need for a preference relation on
arguments, we will show that a dependence on sub-arguments can be
removed by reformulating the set of rules. Suppose that we have
two different arguments for an inconsistency where both arguments
have the same set of last rules.
\[ A_\perp = \{ \langle \{ \langle \{ \langle \emptyset, \alpha \rangle \},
\alpha \leadsto \varphi \rangle \}, \varphi \leadsto \eta \rangle,
\langle \{ \langle \{ \langle \emptyset, \beta \rangle \},
\beta \leadsto \psi \rangle \}, \psi \leadsto \neg \eta \rangle \} \]
\[ \left. \begin{array}{r} \alpha \vdash \alpha \leadsto \varphi
\vdash \varphi \leadsto \eta \\
\beta \vdash \beta \leadsto \psi \vdash \psi \leadsto \neg \eta
\end{array} \right | \hspace{-4pt} - \perp \]
and
\[ A'_\perp = \{ \langle \{ \langle \emptyset, \varphi \rangle \},
\varphi \leadsto \eta \rangle,
\langle \{ \langle \emptyset, \psi \rangle \},
\psi \leadsto \neg \eta \rangle \}. \]
\[ \left. \begin{array}{r} \varphi \vdash \varphi \leadsto \eta \\
\psi \vdash \psi \leadsto \neg \eta \end{array} \right |
\hspace{-4pt} - \perp \]
 Also suppose that $\varphi \leadsto \eta$
must be defeated given $A_\perp$ and $\psi \leadsto \neg \eta$
must be defeated given $A'_\perp$. In the former case, the
situation described by $\alpha$ and $\beta$ represents an
exception on the rule $\varphi \leadsto \eta$. We can, for
example, describe this exception by introducing the rule $\alpha
\wedge \beta \leadsto \neg \eta$ with preference $\alpha \wedge
\beta \leadsto \neg \eta \succ \varphi \leadsto \eta$, or the rule
$\alpha \wedge \beta \leadsto \neg (\varphi \leadsto \eta)$. Since
each of these rules defeats $\varphi \leadsto \eta$, $\varphi
\leadsto \eta$ can no longer defeat $\psi \leadsto \neg \eta$.
Hence, we only have to consider the last rules of an argument for
an inconsistency.

Another possibly problematic situation arises when a set of
arguments supporting a proposition, is stronger than each
individual argument. This is known as {\em accrual of reasons}.
Such a situation suggest that we need to consider preferences
between sets of rules. We can, however, handle such situations by
using a rule that combines the last rules of each argument for
that proposition. To illustrate this, suppose that we have the
following defeasible rules: $\alpha \leadsto \psi$, $\beta
\leadsto \psi$ and $\gamma \leadsto \neg \psi$. Let the last rule
be preferred to the first two rules. Then $\neg \psi$ must hold if
$\gamma$ and either $\alpha$ or $\beta$ hold. By introducing a
rule $\alpha \wedge \beta \leadsto \psi$ and by preferring it to
$\gamma \leadsto \neg \psi$, we can assure that $\psi$ holds
whenever $\alpha$, $\beta$ and $\gamma$ hold.

Another problem arises when a set of arguments for a proposition
weakens the support for the proposition. The approach presented
here offers no solution for such situations. Fortunately, a set of
arguments that weakens the support for a proposition, seems to be
counter-intuitive.

A last motivation for using a preference relation on rules, comes
from Prakken's (1993) investigation of legal argumentation. He
points out that in legal argumentation, meta rules, such as `lex
superior' and `lex posterior', are used to resolve the
inconsistency. These meta rules define a preference relation on
legal norms (the defeasible rules). When arguments disagree, the
meta rules are applied to the last rules of the disagreeing
arguments in order to determine the argument to be defeated.
Prakken illustrates this with legal examples. Also notice that
meta rules can also be subject to defeat in situations where they
specify incompatible relations between rules (Brewka 1994).

From the above discussion, we can draw the following conclusion.
\begin{quote}
Let $\langle \Sigma, D \rangle$ be a defeasible theory and let
$\succ$ be a partial preference relation on $D$. Furthermore, let
\[ A_\perp = \{ \langle A'_1, \eta_1 \leadsto \psi_1 \rangle ,\ldots,
\langle A'_k, \eta_k \leadsto \psi_k \rangle \langle \emptyset,
\sigma_{1} \rangle ,\ldots, \langle \emptyset, \sigma_{j} \rangle
\} \]
be an argument for an inconsistency. So,
\[ \begin{array}{l}
A_1 = \{ \langle A'_1, \eta_1 \leadsto \psi_1 \rangle \}
,...,A_k = \{ \langle A'_k, \eta_k \leadsto \psi_k \rangle \}, \\
A_{k+1} = \{ \langle \emptyset, \sigma_{1} \rangle \},\ldots,
A_{n} = \{ \langle \emptyset, \sigma_{j} \rangle \}
\end{array} \]
are disagreeing arguments.

Then, if $\eta_i \leadsto \psi_i$ is the least preferred last rule
in $\vec{A}_\perp$, $\eta_i \leadsto \psi_i$ must be defeated.
\end{quote}

Since we are using a preference relation on the set of defeasible
rules in order to resolve conflicts, we should extend the
definition of a defeasible theory $\langle \Sigma, D \rangle$ with
the preference relation $\succ$, i.e.\ $\langle \Sigma, D, \succ
\rangle$. Certainly, to describe legal argumentation, this
extension is necessary. If we restrict ourselves to one specific
preference relation, namely {\em specificity}, there is also no need to
extend the definition of a defeasible theory. The {\em
specificity} preference relation  can be derived from the set of
defeasible rules of a defeasible theory.\footnote{Since the set of
rules $D$ is usually considered as background knowledge, we can
determine the specificity preference relation in advance.}

Specificity is the principle by which rules applying to a more
specific situation override  those applying to more general ones.
The most specific situation to which a rule can be applied is the
situation in which only its antecedent is known to hold. In that
situation, its consequent must hold. The following preference
relation is based on the fact that the most specific situation to
which a rule can be applied is the situation in which only its
antecedent is known to hold.

\begin{definition} \label{spec}
Let $K \subseteq L$ be some general background knowledge, 
let $D$ be a set of defeasible rules, and let
$\varphi \leadsto \psi, \eta \leadsto \mu$ be two rules in $D$.

$\varphi \leadsto \psi$ is more specific than $\eta \leadsto \mu$
if and only if, given the premises $\{ \varphi \}$, there is an
argument $A_\eta$ for $\eta$ such that $\bar{A}_{\eta} \subseteq
\{ \varphi \} \cup K$.

$\varphi \leadsto \psi$ is strictly more specific than $\eta
\leadsto \mu$, $\varphi \leadsto \psi \succ_{\rm spec} \eta
\leadsto \mu$, if and only if $\varphi \leadsto \psi$ is more
specific than $\eta \leadsto \mu$ and $\eta \leadsto \mu$ in not
more specific than $\varphi \leadsto \psi$.
\end{definition}

\begin{example}
    Let $\varphi \leadsto \psi$, $\varphi \leadsto \eta$ and
    $\eta \leadsto \neg \psi$ be three defeasible rules.

    Given the premises $\{ \varphi \}$, we can derive the argument
    $A_\eta = \{ \langle \{ \langle \emptyset, \varphi \rangle \},
    \varphi \leadsto \eta \rangle \}$.
    Since $\bar{A}_\eta \subseteq \{ \varphi \} \cup K$,
    $\varphi \leadsto \psi$ is more specific than
    $\eta \leadsto \neg \psi$.
    Furthermore, since, given the premises $\{ \eta \}$, there is no
    argument for $\varphi$, $\varphi \leadsto \psi$ is strictly more
    specific than $\eta \leadsto \neg \psi$.
\end{example}

The above defined specificity preference relation corresponds with
definition of specificity implied by the axioms of conditional
logics (Geffner \& Pearl 1992). This definition of specificity is
relatively weak. Vreeswijk (1991) presents an example showing that
a slightly stronger definition can result in counter intuitive
conclusions.

\section{The belief set} \label{belief-set}
An inconsistency can be resolved considering the last rules of the
argument for the inconsistency. This implies that in case the
inconsistency is resolved, one of these last rules may no longer
justify the belief in its consequent;  i.e.\ the rule is defeated.
For this rule we can construct an argument supporting the {\em
undercutting} defeat of this rule.
\begin{definition} \label{defeat-arg}
Let $A_\perp$ be an argument for an inconsistency and let $\varphi
\leadsto \psi \in min_\succ(\vec{A}_\perp)$ be a least preferred
last rule for the inconsistency.

If $\langle  A_\varphi, \varphi \leadsto \psi \rangle \in
A_\perp$, then $A_{\neg (\varphi \leadsto \psi)} = (A_\perp \backslash \{
\langle A_\varphi, \varphi \leadsto \psi \rangle \} ) \cup
A_\varphi$ is an argument for the defeat of $\varphi \leadsto
\psi$.
\end{definition}

\begin{example}
Let
\[ A_\perp= \{ \langle \{ \langle \{ \langle \emptyset, \alpha \rangle
\}, \alpha \leadsto \varphi \rangle \}, \varphi \leadsto \psi
\rangle, \langle \{ \langle \emptyset, \eta \rangle \}, \eta
\leadsto \mu \rangle \} \]
\[ \left. \begin{array}{r}
\alpha \vdash \alpha \leadsto \varphi \vdash \varphi \leadsto \psi \\
\eta \vdash \eta \leadsto \mu
\end{array} \right | \hspace{-4pt} - \perp \]
If $\eta \leadsto \mu$ is preferred to $\varphi \leadsto \psi$, then
\[ A_{\neg (\varphi \leadsto \psi)} = \{ \langle \{ \langle \emptyset,
\alpha \rangle \}, \alpha \leadsto \varphi \rangle,
\langle \{ \langle \emptyset, \eta \rangle \},
\eta \leadsto \mu \rangle \} \]
\[ \left. \begin{array}{r}
\alpha \vdash \alpha \leadsto \varphi \\
\eta \vdash \eta \leadsto \mu
\end{array} \right | \hspace{-4pt} \circ \neg (\varphi \leadsto \psi) \]
We use the symbol $\mid  \hspace{-5pt} \circ$ to denote that $\neg
(\varphi \leadsto \psi)$ does not deductively follow from
$\varphi$ and $\mu$. $\neg (\varphi \leadsto \psi)$ ``follows''
from $\varphi \leadsto \psi$, $A_\varphi$, $A_\mu$ and $\succ$.
\end{example}

Given these arguments for the defeat of rules, we can define an
{\em extension}. Here an extension is a set of propositions for
which we have valid arguments. A valid argument is an argument of
which the rules are not defeated.
This also holds for
the arguments for the defeat of rules. A rule is defeated if the
argument for its defeat is valid; i.e.\ the argument does not
contain defeated rules.
\begin{definition} \label{defeat}
Let $\cal A$ be a set of all derivable arguments, let $\Gamma$ be
a set of defeasible rules and let
\[ \textsl{Defeat}(\Gamma) = \{ \alpha \leadsto \beta \mid
A_{\neg (\alpha \leadsto \beta)} \in {\cal A},
\tilde{A}_{\neg (\alpha \leadsto \beta)} \cap \Gamma = \emptyset \}. \]

Then the set of defeated rules $\Omega$ is defined as:
\[ \Omega = \textsl{Defeat}(\Omega). \]
\end{definition}
\begin{proposition} \label{unique}
    The set of defeated rules $\Omega$ are incomparable.
    I.e.\ for each $\Lambda \not= \Omega$ such that $\Lambda =
    \textsl{Defeat}(\Lambda)$, neither $\Lambda \subset \Omega$
    nor $\Lambda
    \supset \Omega$ holds.
\end{proposition}
Proofs can be found in Appendix A.

Notice that the set of defeated rules need not be unique. Even if
every inconsistency has a unique least preferred last rule, the
set of defeated rules need not be unique. Consider for example the
facts $\alpha$ and $\beta$ and rules $\alpha \leadsto \gamma$,
$\beta \leadsto \delta$, $\gamma \leadsto \neg \delta$ and $\delta
\leadsto \neg \gamma$, where the last two rules are preferred to
the first two. Here there are two sets of defeated rules $\Omega$;
$\{ \alpha \leadsto \gamma \}$ and $\{ \beta \leadsto \delta \}$.

Given the sets of defeated rules, the extensions and the belief
set can be defined. An extension consists of all propositions for
which we have a valid argument. Following Pollock (1987), these
propositions are said to be {\em warranted}.\footnote{Some
proposals made in the literature do not consider multiple
extensions. Instead, they consider provisionally defeated
arguments. These are arguments that are valid in some extensions
but not in all extensions. For a discussion see (Vreeswijk 1997;
Prakken \& Vreeswijk 1999).}
\begin{definition} \label{argextend}
Let $\Omega$ be a set of defeated rules and let $\cal A$ be a set
of all derivable arguments.

Then an extension $\cal E$ is defined as:
\[ {\cal E} = \{ \varphi \mid A_\varphi \in {\cal A},
\tilde{A}_\varphi \cap \Omega = \emptyset \}. \]
\end{definition}
The belief set contains the propositions that hold in every
extension. This is the skeptical view in which the belief set
consists of every proposition for which we have a valid argument
in every extension.
\begin{definition}
Let $\langle \Sigma, D, \succ \rangle$ be a defeasible theory.
Furthermore, let ${\cal E}_1, ..., {\cal E}_n$ be the corresponding
extensions.

Then the belief set $B$ is defined as:
\[ B = \bigcap^n_{i=1} {\cal E}_i . \]
\end{definition}

It is possible to have a set of arguments for which no extension
exists. Such a situation can arise when the set of arguments
contain self-defeating arguments. In its most simple form,
self-defeat is related to one argument $\alpha \leadsto \beta \in
\tilde{A}_{\neg (\alpha \leadsto \beta)}$.\footnote{Prakken \&
Vreeswijk (1999) present an instance of the liar's paradox as an
example of self-defeat. With a different formulation of the
example, however, we can solve the paradox by using defeasible
rules without introducing self-defeat.} Since the set of
defeasible rules belong to the background knowledge, self-defeat
seems to present a defect in our knowledge. Hence, a revision of
the set of rules $D$ is needed. I.e.\ some rules must be removed or
reformulated. Though this is an important topic, it does not help
us much in practical situations. Hence, we need a way to draw
useful conclusions even if self-defeat occurs. Pollock (1994)
introduces partial status assignments to deal with the problem.
Here, we can do something similar. Firstly, we will reformulate
Definition \ref{defeat} in terms of a status assignment.
\begin{quote}
A status assignment is an assignment of {\em defeated} and {\em
undefeated} to rules in $D$ based on the following condition.

A rule $\varphi \leadsto \psi \in D$ is assigned {\em defeated} if
and only if there is an argument $A_{\neg (\varphi \leadsto
\psi)}$ such that every $\alpha \leadsto \beta \in \tilde{A}_{\neg
(\varphi \leadsto \psi)}$ is assigned {\em undefeated}. Otherwise,
the rule $\varphi \leadsto \psi \in D$ is assigned {\em
undefeated}
\end{quote}
$\Omega$ is the set of rules that are assigned the status {\em
defeated}.
\begin{proposition}
A set of rules $\Omega$ is a fixed point of \textsl{Defeat} if and
only if there is a status assignment such that $\Omega$ is the set
of rules that are assigned the status {\em defeated}.
\end{proposition}

To deal with self-defeat, following Pollock (1994), we can use a
partial status assignment.
\begin{quote}
A {\em partial} status assignment is an assignment of {\em
defeated} and {\em undefeated} to a largest subset of the rules in
$D$ based on the following conditions.
\begin{itemize}
\item
A rule $\varphi \leadsto \psi \in D$ is assigned {\em defeated} if
and only if there is an argument $A_{\neg (\varphi \leadsto
\psi)}$ such that every $\alpha \leadsto \beta \in \tilde{A}_{\neg
(\varphi \leadsto \psi)}$ is assigned {\em undefeated}.
\item
A rule $\varphi \leadsto \psi \in D$ is assigned {\em undefeated}
if and only if for every argument $A_{\neg (\varphi \leadsto
\psi)}$ there is a rule $\alpha \leadsto \beta \in \tilde{A}_{\neg
(\varphi \leadsto \psi)}$ that is assigned {\em defeated}.
\end{itemize}
A rule $\varphi \leadsto \psi \in D$ that is not is assigned {\em
defeated} or {\em undefeated} are denoted as being {\em
undetermined}
\end{quote}
Since we should only consider conclusions based on arguments
containing {\em undefeated} rules, $\Omega$ is the set of rules
that are {\em not} assigned the status {\em undefeated}.

In the remainder of this paper, with exception of the next
section, we will assume that status assignments are complete.

\section{Determination of the fixed point of \textsl{Defeat}} \label{algor}
The determination of the fixed points of \textsl{Defeat} can be
viewed as a labeling problem of a JTMS (Doyle 1979). A JTMS
consists of nodes representing propositions, and of
justifications. A node is either labeled {\sc in} or {\sc out},
which corresponds with respectively `is believed' and `is not
believed'. A justification is a triple consisting of a set of {\em
in}-nodes, a set of {\em out}-nodes and a consequent node. The
consequent node is labeled {\sc in} if all {\em in}-nodes are
labeled {\sc in} and no {\em out}-node is labeled {\sc in}.

Such a JTMS must contain a node $N$ for every proposition of the
form $\neg (\alpha \leadsto \beta)$ for which we have an argument
in $\cal A$. Furthermore, for each node $N_{\neg (\alpha \leadsto
\beta)}$ representing $\neg (\alpha \leadsto \beta)$ and for each
argument in $\cal A$ supporting $\neg (\alpha \leadsto \beta)$,
the JTMS contains a justification $\langle \mbox{{\em in}-nodes},
\mbox{{\em out}-nodes}, N_{\neg (\alpha \leadsto \beta)} \rangle$.
Such a justification consists of an {\em empty set} of {\em
in}-nodes and a set of {\em out}-nodes. If $A$ is an argument for
$\neg (\alpha \leadsto \beta)$, then
\[ \langle \emptyset, \{ N_{\neg(\varphi \leadsto \psi)} \mid
\varphi \leadsto \psi \in \tilde{A} \}, N_{\neg (\alpha \leadsto \beta)}
\rangle \]
is a justification for $N_{\neg (\alpha \leadsto \beta)}$.

It is not difficult to verify that each valid labeling of the
nodes corresponds one to one with status assignment to the
corresponding rules. A rule is assigned the status {\em defeated}
if and only if the corresponding node of the JTMS is labeled {\sc
in}.

Much research has been done on algorithms for labeling nodes in a
JTMS network (Doyle 1979; Goodwin 1987; Reinfrank 1989; Witteveen
\& Brewka 1993). Some also deal with situations in which no
admissible labeling exists (Witteveen \& Brewka 1993). Partial
labeling has been proposed for these situations.

When no admissible labeling exists, the set of arguments $\cal A$
contains self-defeating arguments. In its most simple form,
self-defeat is related to one argument $\alpha \leadsto \beta \in
\tilde{A}_{\neg (\alpha \leadsto \beta)}$. In general, self-defeat
is represented by {\em odd loops} in the corresponding JTMS.

Odd loops in the network can be determined in linear time with
respect to $n \cdot d$ where $n$ is the number of nodes and  $d$
is the maximum number of outgoing justifications of any node.
After detecting an odd loop we can mark the nodes involved as
being `{\sc undetermined}', as well as the nodes that necessarily
depend on nodes in an odd loop. This labeling of some of the nodes
can subsequently be replaced by {\sc in} or {\sc out} if the
labeling of the remaining nodes enforces this. Hopefully, after
labeling all nodes, no {\sc undetermined} nodes remain. By doing
so, we handle odd loops in a pragmatic way.

Finding a labeling of a JTMS network is, in general, an
NP-Hard problem. Fortunately, for the above presented JTMS
networks without {\em odd loops}, we can find a labeling in {\em
linear time} with respect to $n \cdot d + j$ where $n$ is the
number of nodes, $d$ is the maximum number of outgoing
justifications of any node and $j$ is the total number of
justifications. An algorithm for finding a labeling will
be given in Appendix B.
Although a labeling can be found in linear time, the number of 
possible labelings can be exponential in number of minimal 
arguments for inconsistencies if no preference relation over the 
defeasible rules is specified; i.e. $\succ = \emptyset$.

\section{Properties}
\subsubsection{Default logic} \label{default-logic} In Section
\ref{conflicts}, we have seen that it suffices to consider only
the last rules of an argument for an inconsistency. This property
enables us to define a default logic that is equivalent with
respect to the belief set. This default logic will be based on
Brewka's prioritized default logic (Brewka 1994). Brewka argues
that it is sufficient to use only normal default rules in
combination with a preference relation on these rules. Semi-normal
and non-normal default rules are used to realize undercutting
defeat or to define preferences between default rules. Using
semi-normal and non-normal default rules for the encoding of
preferences is not very elegant. A more important problem is,
however, that we cannot specify preferences between default rules
that cause an inconsistency because of contingent information.
Prioritized default logic does not have this drawback. The
prioritized default logic proposed below is similar to Brewka's
prioritized default logic. We will, however, use the preference
relation in a different way as Brewka proposes.

Since we only consider normal default rules, it suffices to verify
whether a rule is applicable --{\em its antecedents hold}--, and
whether it is not defeated by other rules --{\em its consequent
holds}--. The consequences of a set of applicable rules, together
with the premises, may form an inconsistent set of propositions.
Since defeasible rules are viewed as normal default rules, one of
these rules must be defeated. The partial preference relations on
the rules will be used to determine the rule that must be
defeated. If an applicable rule is defeated, there must be a set
of non-defeated applicable rules that implies, together with the
premises, the negation of its consequent. Furthermore, the
defeated rule may not be preferred to any of rules that causes its
defeat.

\begin{definition} \label{extensional}
Let $\langle \Sigma, D, \succ \rangle$ be a defeasible theory.

Let $\Gamma (S)$ be a smallest set, with respect to the inclusion
relation ($\subseteq$), for which the following conditions hold:
\begin{enumerate}
\item
$\Sigma \subseteq \Gamma (S)$;
\item
$\Gamma (S) = Th(\Gamma (S))$;
\item
if there is a $\Delta \subset D$ that defeats $\varphi \leadsto
\psi$ with respect to $\succ$, then $\neg ( \varphi \leadsto \psi
) \in \Gamma (S)$;
\item
if $\varphi \in \Gamma (S)$, $\varphi \leadsto \psi \in D$ and $\neg (
\varphi \leadsto \psi ) \not\in S$,
then $\psi \in \Gamma (S)$.
\end{enumerate}
$\Delta$ defeats $\varphi \leadsto \psi$ with respect to $\succ$
if and only if
\begin{itemize}
    \item  $\varphi \in \Gamma (S)$,

    \item  $\Delta \subseteq \{ \eta \leadsto \mu \in D \mid
    \{ \eta, \mu \} \subseteq \Gamma (S) \}$,

    \item  $\{ \mu \mid \eta \leadsto \mu \in \Delta \} \cup \Sigma
    \vdash \neg \psi$ and

    \item  for no $\eta \leadsto \mu \in \Delta$ there holds:
    $\varphi \leadsto \psi \succ \eta \leadsto \mu$.
\end{itemize}

${\cal E}$ is an extension of the default theory if and only if
${\cal E} = \Gamma({\cal E})$
\end{definition}

\begin{theorem} \label{equivalence}
Let $\langle \Sigma, D, \succ \rangle$ be a defeasible theory.
The set of extensions determined by the argument system is
equal to the set of extensions determined by the default logic.
\end{theorem}

\begin{example}
    Let $\langle \Sigma, D, \succ \rangle$ be a defeasible theory where
    $\Sigma = \{ \alpha, \beta \}$,
    $D = \{ \alpha \leadsto \delta, \beta \leadsto \neg \delta \}$ and
    $\succ = \{ (\beta \leadsto \neg \delta, \alpha \leadsto \delta) \}$.

    Then we can construct the following arguments.
    \[ \begin{array}{l}
    A_{\delta} = \{ \langle \{ \langle \emptyset, \alpha \rangle \},
    \alpha \leadsto \delta \rangle \};  \\
    A_{\neg \delta} = \{ \langle
    \{ \langle \emptyset, \beta \rangle \},
    \beta \leadsto \neg \delta \rangle \};  \\
    A_{\neg (\alpha \leadsto \delta)} = \{ \langle
    \{ \langle \emptyset, \beta \rangle \},
    \beta \leadsto \neg \delta \rangle,
    \langle \emptyset, \alpha \rangle \}
    \end{array} \]
    This set of arguments result in one fixed point,
    $\Omega = \{ \alpha \leadsto \delta \}$ and
    for the function \textsl{Defeat}.
    So, we have an extension
    \[ {\cal E} = Th(\{ \alpha, \beta, \neg \gamma, \neg \delta,
    \neg (\alpha \leadsto \delta) \}) . \]

    According to the default logic given in this section,
    an extension must at least contain the premises
    $\Sigma = \{ \alpha, \beta \}$.

    Suppose now that we cannot defeat $\alpha \leadsto \delta$.
    Then $\delta$ must belong to the extension.
    Furthermore, since
    $\beta \leadsto \neg \delta \succ \alpha \leadsto \delta$,
    $\beta \leadsto \neg \delta$ will not be defeated either.
    Therefore, $\neg \delta$ will belong to the extension.
    But then $\alpha \leadsto \delta$ will be defeated.
    Contradiction.

    Hence, $\alpha \leadsto \delta$ must be defeated.
    Since we cannot defeat $\beta \leadsto \neg \delta$, $\neg \delta$
    will belong to the extension.
    Therefore we can derive $\neg (\alpha \leadsto \delta)$.

    So, we have one extension
    \[ {\cal E}' = Th(\{ \alpha, \beta, \neg \gamma, \neg \delta,
    \neg (\alpha \leadsto \delta) \}) . \]
\end{example}

We can establish a relation between this new prioritized default
logic and Reiter's default logic. Firstly, we can translate
defeasible rules to default rules. Since we must be able to denote
that the application of a default rule is no longer valid, $\neg
(\alpha \leadsto \beta)$, we will associate a name with each
default rule. This name will be used to denote that the rule may
no longer be applied. So if $n_{\alpha \leadsto \beta}$ is the
name of the translation of $\alpha \leadsto \beta$, then $\neg
n_{\alpha \leadsto \beta}$ will be the translation of $\neg
(\alpha \leadsto \beta)$. To ensure that a default rule named
$n_{\alpha \leadsto \beta}$ will not be applied if $\neg n_{\alpha
\leadsto \beta}$ holds, we must use the name of the default rule
as one of the justifications of this default rule. Hence, we
translate a defeasible rule $\alpha \leadsto \beta $ to the
default rule \[ \frac{\alpha : \beta, n_{\alpha \leadsto \beta}}
{\beta}. \] The translations of the defeasible rules are all
semi-normal default rules.

It is not difficult to verify that every extension according to
Definition \ref{extensional} is also a Reiter-extension. Since the
preference relation on the defeasible rules is not taken into
account, some Reiter-extensions need not be extensions according
to Definition \ref{extensional}. To eliminate these extensions, we
must encode the preference relation using default rules. To do
this properly, we must also change the translation of $\alpha
\leadsto \beta$.
\begin{quote}
For every rule $\alpha \leadsto \beta \in D$, introduce a
non-normal default rule:
\[ \frac{\alpha : n_{\alpha \leadsto \beta}}{\beta}. \]

For every set of rules $\{ \eta_1 \leadsto \mu_1,...,\eta_k
\leadsto \mu_k \} \subseteq D$ and for every $\varphi \leadsto
\psi \in D$ such that $\Sigma \cup \{ \mu_1,..., \mu_k \} \vdash
\neg \psi$ and for no $\eta_i \leadsto \mu_i$ there holds: $\varphi \leadsto \psi \succ \eta_i \leadsto \mu_i$, introduce a default rule:
 \[ \frac{\eta_1 \wedge ... \wedge \eta_k : n_{\eta_1 \leadsto \mu_1}
 ,...,n_{\eta_k \leadsto \mu_k}, \neg n_{\varphi \leadsto \psi}}
 {\neg n_{\varphi \leadsto \psi}}. \]
\end{quote}
A disadvantage of this translation is that it depends on the
premises $\Sigma$.

We can also translate default rules to defeasible rules, with the
exception of non-normal default rules. Consider a normal or
semi-normal default rule of the form:
 \[ \displaystyle \frac{\alpha : \beta_1,..., \beta_k, \gamma}{\gamma}. \]
We can translate this default rule to the following defeasible
rules:
 \[ \alpha \leadsto \gamma, \neg \beta_1 \leadsto \neg(\alpha
\leadsto \gamma),..., \neg \beta_k \leadsto \neg(\alpha \leadsto
\gamma). \]

\subsubsection{Specificity}
Poole (1985) gives a semantic definition of specificity based on
the comparison of arguments (theories). His definition does not
use the last rules of an argument as a starting point. Instead,
Poole compares sets of rules. A Poole-argument $\langle D, \alpha
\rangle$ for a proposition $\alpha$ describes a set of rules $D$
needed to derive $\alpha$;  $F_c \cup D \cup F_n \models \alpha$.
Here, the defeasible rules $D$ are represented by implications.
Furthermore, $F_c$ and $F_n$ denote the contingent
and the necessary facts respectively.

According to Poole (1985), $\langle D_1, \psi \rangle$ is more
specific than $\langle D_2, \mu \rangle$, if for every set of possible 
fact $F_p$:
\begin{quote}
if $F_p \cup D_1 \cup F_n \models \psi$ and $F_p \cup D_2 \cup F_n 
\not\models \psi$, then $F_p \cup D_2 \cup F_n \models \mu$.
\end{quote}

In this paper, another definition has been given. This definition
can be related to Poole's definition of specificity.

\begin{theorem}
Let $\varphi \leadsto \psi$ and $\eta \leadsto \mu$ be two rules.

If $\varphi \leadsto \psi$ is more specific than
$\eta \leadsto \mu$ according to Definition \ref{spec}, then there are
two Poole-arguments $\langle D_1, \psi \rangle$ and $\langle D_2, \mu
\rangle$ with $\varphi \leadsto \psi \in D_1$ and $\eta \leadsto \mu \in
D_2$ for which there hold that $\langle D_1, \psi \rangle$ is more
specific than $\langle D_2, \mu \rangle$.
\end{theorem}

The converse of this theorem need not hold. The reason why the
converse need not hold is because a set of rules ($D_1$ or $D_2$)
need not uniquely determine a single argument. Furthermore,
$D_1 = \{ \varphi \leadsto \psi \}$ is more specific than $D_2 = 
\{ \eta \leadsto \mu \}$ according to Poole's definition 
while it is only more specific according to the definition given in this 
paper if there exists an argument $A_\eta$ such that 
$\tilde{A}_\eta \subseteq \{\varphi\}$.

\subsubsection{Closure properties}
Gabbay (1985) has initiated the study of the closure properties of
the non-monotonic derivability relation: `$\nm$' (Gabbay 1985;
Kraus et al. 1990; Makinson 1988). Here, the non-monotonic
derivability relation is defined as:
\begin{list}{}{} \item
$\Sigma \nm \varphi$ if and only if $B$ is the belief set of
$\langle \Sigma, D, \succ \rangle$ and $\varphi \in B$.
\end{list}

Gabbay (1985) argues that there are three axioms that must be
satisfied by all non-monotonic logics.
\begin{description}
    \item[{\it Reflexivity}] \ \\
    if $\varphi \in \Sigma$, then $\Sigma \nm \varphi$;
    \item[{\it Cut}] \ \\
    if $\Sigma \nm \varphi$ and $\Sigma \cup \{\varphi\}
    \nm \psi$, then $\Sigma \nm \psi$;
    \item[{\it Cautious Monotonicity}] \ \\
    if $\Sigma \nm \varphi$ and $\Sigma \nm \psi$,
then $\Sigma \cup \{\varphi\} \nm \psi$;
\end{description}
These axioms characterize the property called {\em cumulativity}.

We wish, of course, that all logical consequences of the set of
premises are also derivable.
\begin{description}
    \item[{\it Deduction}] \ \\
    if $\Sigma \vdash \varphi$, then $\Sigma \nm \varphi$;
\end{description}
This axiom implies {\em Reflexivity}, it implies  together with
{\em Cut} the axiom {\em Right Weakening}, and it implies together
with {\em Cautious Monotonicity} and {\em Cut} the axiom {\em Left
Logical Equivalence}. The latter two axioms have been proposed by
Kraus, Lehmann and Magidor (1990). They also proposed an axiom
characterizing reasoning by cases.
\begin{description}
    \item[{\it Or}] \ \\
    if $\Sigma \cup \{\varphi\} \nm \eta$ and $\Sigma \cup \{\psi\}
\nm \eta$, then $\Sigma \cup \{\varphi \vee \psi\} \nm \eta$;
\end{description}

Non-monotonic logics satisfying {\em Deduction}, {\em Cautious
Monotonicity}, {\em Cut} and {\em Or} are said to belong to system
{\bf P}.

\begin{theorem} \label{closure-prop}
The defeasible theory $\langle \Sigma, D, \succ \rangle$
satisfies: \\ {\em Reflexivity, Deduction, Cut} and, in the absence of odd loops, {\em Cautious Monotony}.

\noindent
An \emph{odd loop} is an odd number of arguments $A_1,\ldots,A_n$ where every $A_{i+1}$ defeats a rule in $\tilde{A}_i$, and $A_1$ defeats a rule in $\tilde{A}_n$.
\end{theorem}

A defeasible theory does not satisfy the closure property {\em
Or}, and therefore does not allow for {\em reasoning by cases}.

\section{Reasoning by cases} \label{RbC}
To enable reasoning by cases in an argument system, the usual
approach is to use indirect argumentation. Indirect argumentation
involves subsidiary arguments that justify a conclusion with
respect to the premises and some assumptions. If we have an
argument for $\varphi$ under the assumption $\alpha$, an argument
for $\varphi$ under the assumption $\beta$ and an argument for
$\alpha \vee \beta$, then we can construct an argument for
$\varphi$ without the assumptions $\alpha$ and $\beta$ using
reasoning by cases. Most argument systems, however, do not allow
for subsidiary argumentation and therefore are not able to reason by
cases. This also holds for the argument system proposed in the
previous sections.

We can of course extend the argument system by (\textit{i}) allowing for
subsidiary arguments and (\textit{ii}) introducing a rule for combining
arguments through reasoning by cases. If arguments were not
defeasible, such a simple extension would suffice. Unfortunately,
arguments are defeasible, so we must also address the
defeasibility of an argument when reasoning by cases. The question
that we have to address is whether an argument can be defeated by
a subsidiary argument when reasoning by cases. If an argument is
defeated by a subsidiary argument in every case described by a
disjunction, then the answer is clearly Yes. Roos (1997a, 1998) has
argued that the answer must also be Yes when an argument is
defeated by only one subsidiary argument corresponding with a case
of a disjunction. He illustrates the necessity for this with the
following example.

Suppose that we have the following rules:
\begin{itemize}
    \item  A person who injures another person must be punished.

    \item  A person who injures another person in self-defense,
    should not be punished.
    \item  A person who is dragged into a fight against his/her will,
    is acting in self-defense.
\end{itemize}
Now suppose that John has injured Peter and that a reliable
witness testifies that either John or Paul has been dragged into
the fight against his will. If the argument for not punishing John
in case he acted in self-defense, would not defeat the argument
for punishing John, we will conclude that John must be punished.
This would be most unfortunate for John if he was dragged into the
fight against his will.

When reasoning by cases, we should be able to apply defeasible rules in a case. The above example suggests that we should also resolve conflicts within the context of a case. 
There is one issue with resolving conflicts within the context of a case, as is illustrated by the following example. 
\begin{quote}
John normally attends a party when he is invited: $ji \leadsto
jp$. Bob and John never attend the same party: $\neg (jp \wedge
bp)$. John is invited to a party: $ji$. 
\end{quote}
The proposition $\neg (jp \wedge bp)$ implies two cases, one in which 
John does not attend the party, and one in which Bob does not. Clearly the 
former case conflicts with the conclusion of the rule $ji \leadsto
jp$. Since facts defeat defeasible rules, we would conclude that John will 
not attend the party, in the former case. This conclusion is not valid 
because the conclusion of the rule $ji \leadsto jp$ is consistent with 
the other case described by the proposition $\neg (jp \wedge bp)$.

To address the above described problem, we will use the following principles for reasoning by cases:
\begin{itemize}
\item
Conclusions drawn in a case may not change when other cases are eliminated  
because of additional information. Of course the overall conclusions may change
because they depend on all cases.
\item
Conflicts must be evaluated using the initial information and the conclusions of applied defeasible rules. The defeasible rules may be applied in a case implied by one or more propositions.
\end{itemize}

A possible way to enable reasoning by case proposed in (Roos 1997a) is by introducing special defeasible rules, called \emph{hypotheses}, that generate cases. To avoid that we consider cases $\alpha$ and $\beta$ that follow from the disjunction $\alpha \vee \beta$, simultaneously, we ensure that the cases are mutually exclusive. Considering cases $\alpha$ and $\beta$ simultaneously corresponds to the case $\alpha \wedge \beta$, which is only one of the three possibilities implied by the disjunction $\alpha \vee \beta$.
The following set of {\em hypotheses} can be used to introduce the
mutually exclusive cases.
\[ 
H = \{  \alpha \vee \beta \leadsto \alpha \wedge \neg \beta,
\alpha \vee \beta \leadsto \alpha \wedge \beta,
  \alpha \vee \beta \leadsto \neg \alpha \wedge \beta \mid
\alpha \vee \beta \in L \; \} 
\]

For every proposition $\psi$ that is unknown with respect to the
partial models, we can derive $\psi \vee \neg \psi$ and $\varphi
\vee \psi$ where $\varphi$ is a known proposition. These
disjunctions should not be considered for reasoning by cases. If
we would, we could defeat the rule `birds fly' using the
disjunction $penguins \vee \neg penguins$ and the rule `penguins
do not fly'. Clearly, we do not want this.

The reason why we should not consider these disjunctions is because
of the difference between {\em unknown} and {\em uncertain}.
Uncertainty is expressed by multiple cases, while unknown
is expressed by a single case of which we do not (yet) know the 
truth-value some atomic propositions. Some of the rules may fill in 
the yet unknown information.

We may apply a disjunction $\alpha \vee \beta$ for reasoning by
cases if $\alpha \vee \beta$ is not a derived tautology and if
$\alpha \vee \beta$ has not been derived from $\alpha$ or $\beta$.
A characteristic of these requirements is that $\alpha \vee \beta$
may not contain more atomic propositions than the proposition from
which it is derived. This requirement blocks the possibility of
introducing irrelevant cases. Furthermore, the only tautologies
that are allowed according to this requirement, cannot do any harm
as we will see below. Hence, we can formulate the following
modified definition of an argument.

\begin{definition} \label{argument2}
{\em Definition \ref{argument} revised.} Let $\langle \Sigma, D,
\succ \rangle$ be a defeasible theory where $\Sigma$ is the set of
premises and $D$ is the set of rules.

Then an argument $A$ for a proposition $\psi$ is recursively
defined in the following way:
\begin{itemize}
\item
For each $\psi \in \Sigma$: $A = \{ \langle \emptyset, \psi
\rangle\}$ is an argument for $\psi$.
\item
Let $A_1, ..., A_n $ be arguments for respectively
$\varphi_1,...,\varphi_n$. If $\varphi_1,...,\varphi_n \vdash
\psi$, then $A = A_1 \cup ... \cup A_n$ is an argument for $\psi$.
\item
For each $\varphi \leadsto \psi \in D$ if $A'$ is an argument for
$\varphi$, then $A = \{ \langle  A', \varphi \leadsto \psi \rangle
\}$ is an argument for $\psi$.
\item
For each $\varphi \leadsto \psi \in H$ if $A'$ is an argument such
that $\hat{A}' \vdash \varphi$ and $At(\varphi) \subseteq
At(\hat{A}')$, then $A = \{ \langle  A', \varphi \leadsto \psi
\rangle \}$ is an argument for $\psi$.
\end{itemize}
The function $At(\cdot)$ denotes the set of atomic propositions used
in a propositions or a set of propositions.
\end{definition}

If we have, for example, and argument for $\alpha \wedge \beta$,
we can derive an argument for $\alpha \vee \neg \alpha$ and for
$\alpha \vee \neg \beta$. It is, however, invalid to apply $\alpha
\vee \neg \alpha$ and $\alpha \vee \neg \beta$ for reasoning by
cases. Fortunately, since any argument for the cases $\neg \alpha$
and $\neg \beta$ will be inconsistent with the argument for
$\alpha \wedge \beta$, we can resolve the problem using
preferences. By preferring any defeasible rule in $D$ to any
hypothesis in $H$, we guarantee that a case described by a
disjunction will not be considered if another case described by
the disjunction is derivable.

\begin{definition}
Let $D$ be a set of defeasible rules and let $H$ be a set of
hypotheses.

For each $\alpha \leadsto \beta \in D$ and for each $\gamma
\leadsto \delta \in (H \backslash D)$: $\alpha \leadsto \beta \succ \gamma
\leadsto \delta$.
\end{definition}

The above definitions allows us to construct an argument for a
case described by a disjunction. The preference relation ensures
that we can apply reasoning by cases if no constituent of a
disjunction is derivable. Furthermore, since the cases introduced
by the hypotheses are mutually exclusive, each case will be
represented by a separate extension. So, disjunctions can be
viewed as describing {\em possible extensions}.

Viewing a disjunction as describing possible extensions is an
important deviation from the `normal' interpretation of a
disjunction. In argument systems multiple extensions arise because
there is no preference between two or more conflicting arguments;
e.g. the Nixon diamond. This can be interpreted as a disjunction
stating that one of the arguments is valid. For each case
described by this disjunction, we create an extension describing
that case.
For real disjunctions we can do the same. We can introduce an
extension for each case described by a disjunction. Above we have
realized this by using hypotheses. These hypotheses create an
extension for each case described by a disjunction.

To illustrate reasoning by cases using defeasible rules, we will
apply the above presented results to the example of John who might
be dragged into a fight.
\[ \begin{array}{l} 
\it John\_injures\_Peter \\ 
\it John\_dragged\_into\_fight \vee Paul\_dragged\_into\_fight \\ 
\it \neg (John\_dragged\_into\_fight \wedge Paul\_dragged\_into\_fight) \\ 
\it John\_dragged\_into\_fight \leadsto self\_defense\_John \\
\it John\_injures\_Peter \leadsto John\_must\_be\_punished \\ 
\it self\_defense\_John \leadsto \neg John\_must\_be\_punished \end{array}
 \]
\begin{small}
 $\it (self\_defense\_John \leadsto \neg
John\_must\_be\_punished) \succ (John\_injures\_Peter \leadsto
John\_must\_be\_punished) $
\end{small}

Using these facts and rules, we can construct arguments. Two of
these argument are:
 \begin{tabbing} mmm\=mmm\=mmm\=mmm\=mmm\=mmm\=mmm\= \kill
 \> $\it A_{John\_must\_be\_punished} =$ \\
 \> \> $\it John\_injures\_Peter \vdash$ \\
 \> \> \> $\it John\_injures\_Peter \leadsto John\_must\_be\_punished \vdash$ \\
 \> \> \> \> $\it John\_must\_be\_punished$ \\ 
 \\
 \> $\it A_{\neg John\_must\_be\_punished} =$ \\
 \> \> $\it John\_dragged\_into\_fight \vee
  Paul\_dragged\_into\_fight \vdash$ \\
 \> \> \> $\it John\_dragged\_into\_fight \vee Paul\_dragged\_into\_fight
  \leadsto$ \\
 \> \> \> \> $\it John\_dragged\_into\_fight \wedge \neg
 Paul\_dragged\_into\_fight \vdash$ \\
 \> \> \> \> \> $\it John\_dragged\_into\_fight \leadsto 
 self\_defense\_John \vdash$ \\
 \> \> \> \> \> \> $\it self\_defense\_John \leadsto \neg
 John\_must\_be\_punished \vdash$ \\
 \> \> \> \> \> \> \>  $\it \neg John\_must\_be\_punished$
\end{tabbing}

Using all derivable arguments, we can determine the following two
extensions. 
\begin{tabbing} mmm\= \kill
\> $\it E_1 = Th(\{ \begin{array}[t]{l}
\it John\_injures\_Peter, \\ 
\it John\_dragged\_into\_fight, \\ 
\it \neg Paul\_dragged\_into\_fight, \\ 
\it self\_defense\_John, \\ 
\it \neg John\_must\_be\_punished, \\
\it \neg (John\_injures\_Peter \leadsto John\_must\_be\_punished) \; \})
\end{array} $ 
\\ \\
\> $\it E_2 = Th(\{ \begin{array}[t]{l}
\it John\_injures\_Peter, \\ 
\it Paul\_dragged\_into\_fight, \\ 
\it \neg John\_dragged\_into\_fight, \\ 
\it John\_must\_be\_punished \; \}) 
\end{array} $
\end{tabbing}
Since in only one of the two situations John must be punished,
we do not know whether John must be punished. Additional
information should be collected to enable us to make a choice
between the two situations that are represented by the two
extensions.

Reasoning by cases does not guarantee that the closure property
{\em Or} holds because cases are mutually exclusive. We can,
however, proof an {\em Exclusive Or} property.
\begin{theorem}
The defeasible theory $\langle \Sigma, D, \succ \rangle$ satisfies
{\em Exclusive Or}: \\
if $\Sigma \cup \{\varphi \wedge \neg\psi\} \nm \eta$, 
$\Sigma \cup \{\neg\varphi \wedge \psi\} \nm \eta$, then $\Sigma
\cup \{\varphi \xor \psi \} \nm \eta$;
\end{theorem}

\section{Related work} \label{RL}
In the literature, several argument systems that apply defeasible
rules have been proposed (Fox et al. 1992; Geffner 1994; Krause et
al. 1995; Pollock 1987; Prakken 1993; Prakken \& Vreeswijk 1999;
Simari \& Loui 1992; Vreeswijk 1991; Vreeswijk 1997). These
related papers can roughly be divided in three groups; those that
discuss the strength of an argument (Fox et al. 1992; Krause et
al. 1995), those that discuss the evaluation of arguments
supporting conflicting propositions (Geffner 1994; Prakken 1993;
Simari \& Loui 1992; Vreeswijk 1991) and those that discuss
validity of arguments (Pollock 1987; Pollock 1994; Prakken \&
Vreeswijk 1999; Simari \& Loui 1992; Vreeswijk 1997).

Krause, Ambler, Elvang-G{\o}ransson and Fox (Fox et al. 1992;
Krause et al. 1995) present argument systems that enable us to
determine the strength of an argument for a proposition They use a
simple logic consisting of atoms, including $\perp$, and Horn
clauses. For this logic they develop an argument system that
enables them to evaluate the strength of arguments for a
consistent set of propositions probabilistically\footnote{A rule
is not interpreted as representing a conditional probability.}.
Furthermore, the argument system enables them to evaluate the
strength of arguments for an inconsistent set of propositions
symbolically. Krause et al.
do not, however, discuss how to defeat one of the disagreeing
arguments.

Closely related to the strength of an argument is the evaluation
of disagreeing arguments that support an inconsistency. Simari and
Loui (1992) have proposed to apply Poole's definition of
specificity for this purpose. In this definition it is necessary
to consider all the rules of the disagreeing arguments. The same
approach is taken by Prakken (1993). Prakken argues that in legal
argumentation only the last rules of an argument for an
inconsistency are considered. In case of specificity, however, he
uses Poole definition.

Vreeswijk (1991) discusses some general principles to evaluate
disagreeing arguments. He proposes a scheme for evaluating
disagreeing arguments based on the types of these arguments. He
derives these types from the structure of the arguments.
Furthermore, he argues that beside these weak general principles
there are no general guild-lines to evaluate arguments. The
definition of specificity given in Section \ref{evaluate}
correspond with the general principles of Vreeswijk. However,
applying it as a preference relation does not.

Geffner (1994) argues that any rule of an argument for a
proposition $\varphi$ can be defeated if $\neg \varphi$ is a known
fact. As we have seen in Section \ref{argument-sys}, Geffner uses
causal rules which are a special kind of defeasible rules. These
kind of rules have been excluded from this paper. Many defeasible
rules are not causal rules. Furthermore, a discussion of causal
rules would also require a study of causality.

The theory of warrant is concerned with the validity of arguments.
These are the arguments that are not defeated by other arguments.
In (Pollock 1987), Pollock introduces the theory of inductive
warrant. Simari and Loui (1992) combine the theory of inductive
warrant with Poole's definition of specificity and study the
mathematical properties of the resulting system.

Pollock (1990) observes that his theory of inductive warrant is
not without problems. Therefore, he introduces a new theory of
warrant based on the idea of multiple extensions. Vreeswijk (1991)
has made a similar proposal.

In (Vreeswijk 1997), Vreeswijk relates the theory of inductive
warrant to a theory of warrant based on extensions. He discusses
several ways of defining a theory of warrant and discusses the
mutual relation.

Dung (1995) discusses the theory of warrant on an abstract level.
He presents several notions of acceptable arguments based on a set
of arguments and a binary attack relation on the set of arguments.
Here, arguments are considered as atomic entities. No relation
with an argument system is specified. The different notions of
acceptability correspond with different ways of dealing with
self-defeat and with multiple extensions. The extensions based on
Definition \ref{defeat} correspond with Dung's {\em stable
extensions} and the extensions based on the partial status
assignment correspond with Dung's {\em preferred extensions}.

Prakken and Vreeswijk (1999) give an overview of argument systems
proposed in the literature. In their overview they discuss the
strong and weak points of many theories of warrant that have been
proposed in the literature. One of the aspect they look at is the
handling of self-defeat. Furthermore, they discuss several
arguments systems in detail.

The theories of warrant presented by Pollock (1987, 1994), Dung
(1995), Vreeswijk (1997), and Prakken and Vreeswijk (1999), start
from a {\em defeat relation} on the set of derived arguments. This
relation is the result of resolving conflicts between the
propositions supported by the arguments. The theory of warrant as
described by these authors is concerned with selecting a set of
valid arguments. It seems more natural, however, to express the
validity of an argument in terms of the validity of the defeasible
steps that are used in the argument. In this respect, the theory
of warrant proposed in this paper differs from the above mentioned
proposals.

Nute (1988, 1994) proposes a {\em defeasible logic} that is
closely related to argument systems for reasoning with defeasible
rules. Nute's logic, which seems to be inspired by logic
programming, does not derive arguments that are subsequently
evaluated to determine the valid conclusions. Instead, Nute
introduces a proof system that guarantees that only valid
conclusions are derived. The proof system consists of four rules
for deriving formulas that hold and three rules for formulas that
cannot hold (Nute 1994). Since the preference relation
`specificity' is an integral part of these rules, the formulation
of the rules is rather complex. Furthermore, the approach is less
flexible. Adding other preference relations requires a
reformulation of the rules.

Conclusions that follow form the defeasible logic can be weaker
than the conclusions one would expect. Ideally, in case of
multiple extensions, conclusions should be based on those
arguments that are valid in all extensions. Prakken \& Vreeswijk
(1999) point out that in Nute's defeasible logic, this is not
always the case. Nute's approach sometimes allows for a smaller
number of conclusions than necessary.

An advantage of Nute's approach is its suitability for realizing
an implementation. His approach gives us a recursive procedure for
the determination of validity of a conclusion. This is in contrast
with the procedure proposed by Loui (1998). Loui views the
procedure of determining the validity of a conclusion as a
dialectics satisfying some protocol.

\section{Conclusion}
A defeasible rule describes a preferred or a probabilistic
relations between propositions. Such defeasible rules can be used
to construct arguments for propositions. For both interpretations
of a defeasible rule, we conclude that an inconsistency can be
resolved by defeating one of the last rules of the argument
supporting the inconsistency. Furthermore, we conclude that it
suffices to consider only the rules are candidates for defeat, to
select the rule to be defeated. For this purpose, a preference
relation on the set of rules has been proposed. A definition of
{\em specificity} that generates such a preference relation on the
set of rules has been given.


Since one of the last rules of the argument for an inconsistency
must be defeated, we can formulate an argument for the defeat of
this rule. Such an argument {\em undercuts} the application of the
rule. Hence, {\em rebutting} defeat is reformulated as {\em
undercutting} defeat after determining the rule to be defeated.
Although this approach does not lead to new results, it is more
intuitive. An argument gives a valid justification for a
conclusion, if all step (the rules) of the justification are
valid. Furthermore, it enables us to determine the extensions of
valid beliefs using a Reason Maintenance System.

A relation between default logic and the proposed argument system
has been established and closure properties have been studied.
Finally, an extension of the argument system enabling reasoning by
cases has been proposed.

\section*{Appendix A}
\setcounter{theorem}{0} \setcounter{proposition}{0}
\begin{proposition}
    The set of defeated rules $\Omega$ are incomparable.
    i.e.\ for each $\Lambda \not= \Omega$ such that $\Lambda =
    \textsl{Defeat}(\Lambda)$, neither $\Lambda \subset \Omega$
    nor $\Lambda
    \supset \Omega$ holds.
\end{proposition}
\begin{proof}
    Suppose $\Lambda \subset \Omega$.
    Then, by the definition of \textsl{Defeat}: $\textsl{Defeat}(\Lambda)
    \supseteq \textsl{Defeat}(\Omega)$.
    Hence, $\Lambda \supseteq \Omega$.
    Contradiction.

    Suppose $\Omega \subset \Lambda$.
    Then, by the definition of \textsl{Defeat}: $\textsl{Defeat}(\Omega)
    \supseteq \textsl{Defeat}(\Lambda)$.
    Hence, $\Omega \supseteq \Lambda$.
    Contradiction.
\end{proof}

\begin{proposition}
A set of rules $\Omega$ is a fixed point of \textsl{Defeat} if and
only if there is a status assignment such that $\Omega$ is the set
of rules that is assigned the status {\em defeated}.
\end{proposition}
\begin{proof}
Let $\Omega = \{ \varphi \leadsto \psi \mid \varphi \leadsto \psi
\mbox{ is assigned the status {\em defeated}} \}$.

Suppose that $\Omega$ is not a fixed point. Then, for some
$\varphi \leadsto \psi$, $\varphi \leadsto \psi \in
\textsl{Defeat}(\Omega)$ and $\varphi \leadsto \psi \not\in
\Omega$ or $\varphi \leadsto \psi \not\in \textsl{Defeat}(\Omega)$
and $\varphi \leadsto \psi \in \Omega$.

Suppose that $\varphi \leadsto \psi \in \textsl{Defeat}(\Omega)$
and $\varphi \leadsto \psi \not\in \Omega$. Since $\varphi
\leadsto \psi \in \textsl{Defeat}(\Omega)$, there is an argument
$A_{\neg(\varphi \leadsto \psi)}$ such that $A_{\neg(\varphi
\leadsto \psi)} \cap \Omega = \emptyset$. But then, by the
definition of a status assignment, $\varphi \leadsto \psi$ is
assigned the status {\em defeated}. Contradiction.

Suppose that $\varphi \leadsto \psi \not\in
\textsl{Defeat}(\Omega)$ and $\varphi \leadsto \psi \in \Omega$.
Since $\varphi \leadsto \psi \not\in \textsl{Defeat}(\Omega)$,
there is no argument $A_{\neg(\varphi \leadsto \psi)}$ such that
$A_{\neg(\varphi \leadsto \psi)} \cap \Omega = \emptyset$. But
then, by the definition of a status assignment, $\varphi \leadsto
\psi$ is assigned the status {\em undefeated}. Contradiction.

Hence, $\Omega$ is a fixed point.

Now let $\Omega$ be a fixed point of \textsl{Defeat}.
We assign the status {\em defeated} to all rules in $\Omega$ and 
{\em undefeated} to all rules not in $\Omega$.

Suppose that this is not a valid status assignment. Then there is 
a rule $\varphi \leadsto \psi$ that is assigned the status {\em defeated}
while there is no argument $A_{\neg(\varphi \leadsto \psi)}$ such
that every $\alpha \leadsto \beta \in \tilde{A}_{\neg(\varphi
\leadsto \psi)}$ is assigned the status {\em undefeated}, or
$\varphi \leadsto \psi$ that is assigned the status {\em
undefeated} while there in an argument $A_{\neg(\varphi \leadsto
\psi)}$ such that every $\alpha \leadsto \beta \in
\tilde{A}_{\neg(\varphi \leadsto \psi)}$ is assigned the status
{\em undefeated}.

In the former case, for every argument $A_{\neg(\varphi \leadsto
\psi)}$ there is a rule $\alpha \leadsto \beta \in
\tilde{A}_{\neg(\varphi \leadsto \psi)}$ that is assigned the
status {\em defeated}. Therefore, for every argument 
$A_{\neg(\varphi \leadsto \psi)}$, $\tilde{A}_{\neg(\varphi
\leadsto \psi)} \cap \Omega \not= \emptyset$. Hence $\varphi
\leadsto \psi \not\in \Omega$ and therefore $\varphi \leadsto
\psi$ is assigned the status {\em undefeated}. Contradiction.

In the latter case, there is an argument $A_{\neg(\varphi \leadsto
\psi)}$ such that every $\alpha \leadsto \beta \in
\tilde{A}_{\neg(\varphi \leadsto \psi)}$ is assigned the status
{\em undefeated}. Therefore, $\tilde{A}_{\neg(\varphi \leadsto
\psi)} \cap \Omega = \emptyset$. Hence $\varphi \leadsto \psi \in
\Omega$ and therefore $\varphi \leadsto \psi$ is assigned the
status {\em defeated}. Contradiction.
\end{proof}

To prove Theorem \ref{equivalence}, the following lemmas will be
used.

\begin{lemma} \label{l3}
Let $\Gamma$ be a set of propositions and let $\cal E$ be the
deductive closure of $\Gamma$. Furthermore, let there be an
argument for each proposition in $\Gamma$.

Then for each proposition in $\cal E$ we can construct an argument $A$.
\end{lemma}

\begin{proof}
For each $\varphi \in {\cal E} \backslash \Gamma$ there holds that
$\Gamma \vdash \varphi$.
Hence, $\bigcup \{ A_\psi \mid \psi \in \Gamma \}$ is an argument for
$\varphi$.
\end{proof}

\begin{lemma} \label{l4}
Let $\cal E$ be an extension according to Definition
\ref{extensional} and let $\Omega = \{ \alpha \leadsto \beta \mid
\neg (\alpha \leadsto \beta) \in {\cal E} \}$. Furthermore, let
there be an argument $A$ for each proposition in $\cal E$ and let
$\tilde{A} \cap \Omega = \emptyset$.

Then $\Omega$ satisfies Definition \ref{defeat}, $\Omega =
\textsl{Defeat}(\Omega)$.
\end{lemma}

\begin{proof}
Suppose that $\alpha \leadsto \beta \in \Omega$ and $\alpha
\leadsto \beta \not\in \textsl{Defeat}(\Omega)$.

Since $\alpha \leadsto \beta \in \Omega$, either there exists a
$\gamma \leadsto \neg ( \alpha \leadsto \beta)$ and $\gamma \in
{\cal E}$, or there exists a $\Delta$ that defeats $\alpha
\leadsto \beta$, $\alpha \in {\cal E}$ and $\Delta \subseteq \{
\eta \leadsto \mu \in D \mid \{ \eta, \mu \} \subseteq {\cal E}
\}$ such that $\{ \mu \mid \eta \leadsto \mu \in \Delta \} \cup
\Sigma \vdash \neg \beta$ and for no $\eta \leadsto \mu \in
\Delta$ there holds: $\alpha \leadsto \beta \succ \eta \leadsto
\mu$.

In the former case there exists an argument $A_{\neg(\alpha
\leadsto \beta)} = \{ \langle  A_\gamma, \gamma \leadsto \neg (
\alpha \leadsto \beta) \rangle \}$ and $\tilde{A}_{\neg(\alpha
\leadsto \beta)} \cap \Omega = \emptyset$. Hence, $\alpha \leadsto
\beta \in \textsl{Defeat}(\Omega)$. Contradiction.

In the latter case there exists an argument $A_{\neg(\alpha
\leadsto \beta)} = \{ \langle  A_\eta, \eta \leadsto \mu \rangle
\mid \eta \leadsto \mu \in \Delta \} \cup \{ \langle \emptyset,
\varphi \rangle \mid \varphi \in \Sigma \} \cup A_\alpha$.
Furthermore, $\tilde{A}_{\neg(\alpha \leadsto \beta)} \cap \Omega
= \emptyset$. Hence, $\alpha \leadsto \beta \in
\textsl{Defeat}(\Omega)$. Contradiction.

Hence, $\Omega \subseteq \textsl{Defeat}(\Omega)$.

Suppose that $\alpha \leadsto \beta \not\in \Omega$ and $\alpha
\leadsto \beta \in \textsl{Defeat}(\Omega)$. Then there exists an
argument $A_{\neg (\alpha \leadsto \beta)}$ such that $A_{\neg
(\alpha \leadsto \beta)} \cap \Omega = \emptyset$. This implies
that either there exists an argument $A_{\alpha}$ for $\alpha$ 
such that $A_{\alpha} \cap \Omega = \emptyset$ and an argument
$A_{\neg \beta}$ for $\neg
\beta$ such that $A_{\neg \beta} \cap \Omega = \emptyset$, or that
$\gamma \leadsto \neg (\alpha \leadsto \beta) \in D$ and there
exists an argument $A_\gamma$ for $\gamma$ such that $A_\gamma
\cap \Omega = \emptyset$.

In the former case, $\alpha \in {\cal E}$ and $\neg \beta \in {\cal E}$.
But then $\neg (\alpha \leadsto \beta) \in \mathcal{E}$.
Contradiction.

In the latter case, $\gamma \in {\cal E}$.
But then $\neg (\alpha \leadsto \beta) \in \mathcal{E}$.
Contradiction.

Hence, $\Omega = \textsl{Defeat}(\Omega)$.
\end{proof}

\begin{theorem}
Let $\langle \Sigma, D, \succ \rangle$ be a defeasible theory.
The set of extensions determined by the argument system is
equal to the set of extensions determined by the default logic.
\end{theorem}

\begin{proof}
Let $\cal E$ be an extension according to Definition
\ref{argextend}. We will proof that ${\cal E}$ is also an
extension according to Definition \ref{extensional} by showing
that it is a fixed point satisfying the four requirements of
Definition \ref{extensional}; i.e., we assume that \ ${\cal E} = \Gamma({\cal E})$.
\begin{enumerate}
\item
Clearly for each $\alpha \in \Sigma$ we have an argument $\{ \langle
 \emptyset, \alpha \rangle \}$.
Since it contains no rules, it cannot be defeated. Therefore,
$\Sigma \subseteq {\cal E}$.
\item
According to the definition of an argument, $\cal E$ is
deductively closed.
\item
Let $\Delta = \{ \eta_1 \leadsto \mu_1,...,\eta_n \leadsto \mu_n
\}$ defeat $\alpha \leadsto \beta$ given $\cal E$. Then $\{
\mu_1,..., \mu_n, \alpha \} \subseteq \Gamma({\cal E}) = {\cal E}$ and for no $\eta_i
\leadsto \mu_i \in \Delta$: $\alpha \leadsto \beta \succ \eta_i
\leadsto \mu_i$. Since $\{ \mu_1,..., \mu_n, \alpha \} \subseteq
{\cal E}$, we have valid arguments $A_{\mu_1},...,A_{\mu_n},
A_\alpha$. Hence, we have a valid argument $A_{\neg (\alpha
\leadsto \beta)}$ for $\neg (\alpha \leadsto \beta)$. Therefore,
$\neg (\alpha \leadsto \beta) \in {\cal E}$.
\item
Let $\alpha \in \Gamma({\cal E}) = {\cal E}$ and $\neg (\alpha \leadsto \beta)
\not\in {\cal E}$. Then there exists a valid argument $A_\beta =
\{ \langle A_\alpha, \alpha \leadsto \beta \rangle \}$ for
$\beta$. Hence, $\beta \in {\cal E} = \Gamma({\cal E})$.
\end{enumerate}
Hence, $\Gamma ({\cal E}) \subseteq {\cal E}$.

Suppose that ${\cal E}$ is not a minimal set satisfying the
requirements of $\Gamma({\cal E})$. Then there is a $\varphi \in
{\cal E} \backslash \Gamma({\cal E})$ and a corresponding valid argument
$A_\varphi$. Let $A_\psi$ be the smallest sub-argument such that
$\psi \not\in \Gamma({\cal E})$.

Suppose that $A_\psi = \{ \langle \emptyset, \psi \rangle \}$.
Since $\psi \in \Sigma$, $\psi \in \Gamma({\cal E})$.
Contradiction.

Suppose that $\varphi_1,...,\varphi_n \vdash \psi$ and $\{
A_{\varphi_1},...,A_{\varphi_n} \} \subseteq {\cal A}$. Then,
since $A_\psi$ is the smallest sub-argument, $\{
\varphi_1,...,\varphi_n \} \subseteq \Gamma({\cal E})$. Therefore
$\psi \in \Gamma({\cal E})$. Contradiction.

Suppose that $A_\psi$ with $\psi = \neg(\eta_i \leadsto \mu_i)$ is
the result of an argument $A_\perp = \{ \langle A'_1, \eta_1
\leadsto \mu_1 \rangle ,\ldots, \langle A'_k, \eta_k \leadsto
\mu_k \rangle \langle \emptyset, \sigma_{1} \rangle ,\ldots,
\langle \emptyset, \sigma_{\ell} \rangle \}$ for an inconsistency.
Clearly for no $j \not = i$: $\eta_j \leadsto \mu_j \succ \eta_i
\leadsto \mu_i$. Since $A_\psi$ is the smallest sub-argument, $\{
\mu_1,...,\mu_{i-1},\mu_{i+1},...,\mu_k,
\sigma_{1},...,\sigma_\ell \} \subseteq \Gamma({\cal E})$.
Therefore $\psi = \neg(\eta_i \leadsto \mu_i) \in \Gamma({\cal
E})$. Contradiction.

Suppose that $A_\psi = \{ \langle A_\mu, \eta \leadsto \psi
\rangle \}$. Since $\cal E$ is an extension according to
Definition \ref{argextend}, there is no valid argument for $\neg
(\eta \leadsto \psi)$. Therefore, $\neg (\eta \leadsto \psi)
\not\in {\cal E}$. Hence, $\psi \in \Gamma({\cal E})$.
Contradiction.

Hence, ${\cal E}$ is a fixed point of $\Gamma$.

\vspace{2mm}

Let $\cal E$ be an extension according to Definition
\ref{extensional} and let $\Omega = \{ \alpha \leadsto \beta \mid
\neg (\alpha \leadsto \beta) \in {\cal E} \}$. 
So, ${\cal E} = \Gamma({\cal E})$. We will proof that
${\cal E}$ is an extension according to Definition \ref{argextend}
by showing that for each proposition in ${\cal E}$ there is a
valid argument and for each proposition not in ${\cal E}$ there is
no such argument. We will show  that there is a valid argument for
each $\varphi \in {\cal E}$ by showing that we can construct an
argument $A$ for each $\varphi \in {\cal E}$ such that $\tilde{A}
\cap \Omega = \emptyset$. If we have an argument for each $\varphi
\in {\cal E}$, then, by Lemma \ref{l4}, $\Omega$ satisfies
Definition \ref{defeat}, i.e.\ $\Omega = \textsl{Defeat}(\Omega)$.
Since for each $\varphi \in {\cal E}$, we have an argument $A$
such that $\tilde{A} \cap \Omega = \emptyset$, $A$ must be a valid
argument for $\varphi$

Let $\Gamma_0 = \Sigma$ and let ${\cal E}_0 \subseteq {\cal E}$ be
a smallest deductively closed subset such that $\Gamma_0 \subseteq
{\cal E}_0$. For each $\varphi \in \Gamma_0$ we can construct an
argument $A_\varphi = \{ \langle \emptyset, \varphi \rangle \}$.
Furthermore, by Lemma \ref{l3}, we can construct an argument for
each $\varphi \in {\cal E}_0$. Clearly, $\tilde{A}_\varphi \cap
\Omega = \emptyset$.

Proceeding inductively, let ${\cal E}_i \subseteq {\cal E}$ be a
smallest deductively closed subset such that $\Gamma_i \subseteq
{\cal E}_i$. Suppose that ${\cal E}_i \subset {\cal E}$. Then
there is a $\varphi \in ({\cal E} \backslash {\cal E}_i)$ such that either
$\varphi = \neg (\alpha \leadsto \beta)$ and $\alpha \leadsto
\beta$ is defeated given ${\cal E}_i$, or $\alpha \in {\cal E}_i$,
$\alpha \leadsto \varphi \in D$ and $\neg (\alpha \leadsto
\varphi) \not\in {\cal E}$, or neither of these two possibilities.

In the third case, $\Gamma({\cal E})$ is not a minimal set. Hence, 
this case is impossible.

In the first case there is a $\Delta \subseteq \{ \eta \leadsto
\mu \in D \mid \{ \eta,\mu \} \subseteq {\cal E}_i \}$ that
defeats $\alpha \leadsto \beta$. Hence we can construct an
argument $A_{\neg(\alpha \leadsto \beta)}$ for $\varphi$ such that
$\tilde{A}_{\neg(\alpha \leadsto \beta)} \cap \Omega = \emptyset$.

In the second case $A_\varphi = \{ \langle A_\alpha, \alpha
\leadsto \varphi \rangle \}$ is an argument for $\varphi$ such
that $\tilde{A}_\varphi \cap \Omega = \emptyset$.

Let ${\cal E}_{i+1}$ be the deductive closure of $\Gamma_{i+1} =
\Gamma_i \cup \{ \varphi \}$. According to Lemma \ref{l3}, for
every proposition in ${\cal E}_{i+1}$ we can construct an argument
$A$ such that $\tilde{A} \cap \Omega = \emptyset$.

Hence, for every proposition in ${\cal E}$ we can construct an
argument $A$ such that $\tilde{A} \cap \Omega = \emptyset$. Given
these arguments, there holds according to Lemma \ref{l4} that
$\Omega = \{ \alpha \leadsto \beta \mid \neg (\alpha \leadsto
\beta) \in {\cal E} \}$ satisfies Definition \ref{defeat}, i.e.
$\Omega = \textsl{Defeat}(\Omega)$. Hence, the arguments for the
propositions in ${\cal E}$ are valid arguments.

\vspace{2mm}

Now suppose that we can construct a valid argument $A_\varphi$ for
a proposition $\varphi \not\in {\cal E}$, i.e., $\tilde{A}_\varphi \cap 
\Omega = \emptyset$. Since $\bar{A}_\varphi
\subseteq \Sigma$, either for some rule $\alpha \leadsto \beta \in
\tilde{A}_\varphi$ there holds: $\alpha \in {\cal E}$ and $\beta
\not\in {\cal E}$. So, $\alpha \leadsto \beta \in \Omega$.
Contradiction.

Hence, ${\cal E}$ is an extension according to Definition
\ref{argextend}.
\end{proof}

\begin{theorem}
Let $\varphi \leadsto \psi$ and $\eta \leadsto \mu$ be two rules.

If $\varphi \leadsto \psi$ is more specific than $\eta \leadsto
\mu$ according to Definition \ref{spec}, then there are two
Poole-arguments $\langle D_1, \psi \rangle$ and $\langle D_2, \mu
\rangle$ with $\varphi \leadsto \psi \in D_1$ and $\eta \leadsto
\mu \in D_2$ for which there hold that $\langle D_1, \psi \rangle$
is more specific than $\langle D_2, \mu \rangle$.
\end{theorem}

\begin{proof}
We must prove that for every set of possible facts $F_p$ there
must hold: if $F_p \cup D_1 \cup F_n \models \psi$ and $F_p \cup
D_2 \cup F_n \not\models \psi$, then $F_p \cup D_2 \cup F_n
\models \mu$. Let $D_2 = \tilde{A}_\eta \cup \{ \eta \leadsto \mu
\}$ and $D_1 = \{ \varphi \leadsto \psi \}$.

Since $\varphi \leadsto \psi$ is more specific than $\eta \leadsto
\mu$, given the premise $\{ \varphi \}$ there must exist an
argument $A_\eta$ for $\eta$.

Suppose that $\bar{A}_\eta = \emptyset$. Then $\langle D_1, \psi
\rangle$ is more specific than $\langle D_2, \mu \rangle$.

Suppose that $\bar{A}_\eta = \{ \varphi \}$. Then, any possible
fact $F_p$ for which the antecedent of Poole's definition holds,
must imply $\varphi$. Hence, $\langle D_1, \psi \rangle$ is more
specific than $\langle D_2, \mu \rangle$.
\end{proof}

\begin{theorem}
The defeasible theory $\langle \Sigma, D, \succ \rangle$
satisfies: \\ {\em Reflexivity, Deduction, Cut} and, in the absence of odd loops, {\em Cautious Monotony}.

\noindent
An \emph{odd loop} is an odd number of arguments $A_1,\ldots,A_n$ where every $A_{i+1}$ defeats a rule in $\tilde{A}_i$, and $A_1$ defeats a rule in $\tilde{A}_n$.
\end{theorem}

\begin{proof}
{\em Reflexivity}. For each $\varphi \in \Sigma$, $A= \{ \langle
\emptyset, \varphi \rangle \}$ is an argument for $\varphi$. Since
$A$ contains no rule, it cannot be defeated. Therefore, $\varphi
\in B$.

{\em Deduction}. For each $\varphi$ such that $\Sigma \vdash
\varphi$, $A= \{ \langle \emptyset, \psi \rangle \mid \psi \in
\Sigma \}$ is an argument for $\varphi$. Since $A$ contains no
rules, it cannot be defeated. Therefore, $\varphi \in B$.

{\em Cut}. Let $\cal E$ be an
extension of the defeasible theory $\langle \Sigma, D, \succ
\rangle$, let $B$ be the belief set of $\langle \Sigma, D, \succ
\rangle$, and let $\varphi \in B$ 

Suppose that $\cal E$ is no longer an extension after
adding some $\varphi \in B$ to $\Sigma$. Let $\Omega$ be
the set of defeated rules that correspond with the extension $\cal
E$. Then after adding $\varphi$ there must be a new argument $A_{\neg
(\alpha \leadsto \beta)}$ such that $\tilde{A}_{\neg (\alpha
\leadsto \beta)} \cap \Omega = \emptyset$ and $\alpha \leadsto
\beta \not\in \Omega$. Since $\tilde{A}_{\neg (\alpha \leadsto
\beta)} \cap \Omega = \emptyset$ and $\alpha \leadsto \beta
\not\in \Omega$, $\{ \langle \emptyset, \varphi \rangle \}$ must
be a sub-argument of $A_{\neg (\alpha \leadsto \beta)}$

Now three situations are possible.
\begin{itemize}
\item
$A_{\neg (\alpha \leadsto \beta)} = \{ \langle A_\xi, \xi \leadsto
\neg (\alpha \leadsto \beta) \rangle \}$. Then there is an
$A^*_{\neg (\alpha \leadsto \beta)}$ in which $\{ \langle
\emptyset, \varphi \rangle \}$ is replaced by $A_\varphi$. Since
$A_\varphi$ is valid; i.e.\ $\tilde{A}_\varphi \cap \Omega =
\emptyset$, there holds that $\alpha \leadsto \beta \in \Omega$.
Contradiction.
\item
$A_{\neg (\alpha \leadsto \beta)}$ is derived from $A_\perp$ and
$\{ \langle \emptyset, \varphi \rangle \}$ is not a disagreeing
argument. Since $\{ \langle \emptyset, \varphi \rangle \}$ is a
sub-argument of $A_\perp$, there is an $A^*_\perp$ in which $\{
\langle \emptyset, \varphi \rangle \}$ is replaced by $A_\varphi$.
Clearly, $\vec{A}_\perp = \vec{A}^*_\perp$. Hence, since
$A_\varphi$ is valid, $\alpha \leadsto \beta \in \Omega$.
Contradiction.
\item
$A_{\neg (\alpha \leadsto \beta)}$ is derived from $A_\perp$ and
$\{ \langle \emptyset, \varphi \rangle \}$ is a disagreeing
argument. Then there is an $A^*_\perp$ in which $\{ \langle
\emptyset, \varphi \rangle \}$ is replaced by $A_\varphi$. Hence,
$A_\varphi \subseteq A^*_\perp$. Since $A_\varphi$ is valid, for no
$\eta \leadsto \mu \in \vec{A}^*_\perp$: $\alpha \leadsto \beta
\succ \eta \leadsto \mu$. Therefore, there is an $A^*_{\neg (\alpha
\leadsto \beta)} = (A^*_\perp \backslash \{ \langle A_\alpha, \alpha
\leadsto \beta \rangle \}) \cup A_\alpha$ and $A^*_{\neg (\alpha
\leadsto \beta)} \cap \Omega = \emptyset$. Hence, $\alpha \leadsto
\beta \in \Omega$. Contradiction.
\end{itemize}

{\em Cautious Monotonicity}.
Let $\langle \Sigma, D, \succ \rangle$ be a defeasible theory, and let $B$ be the belief set of $\langle \Sigma, D, \succ \rangle$.

Suppose that $\cal E$ is an extension of the defeasible theory $\langle \Sigma \cup \{\varphi\}, D, \succ \rangle$ for some $\varphi \in B$, but not of $\langle \Sigma, D, \succ \rangle$. Let $\Omega$ be the set of defeasible rules determining the extension $\mathcal{E}$. Every extension $\mathcal{E}'$ of $\langle \Sigma, D, \succ \rangle$ determined by the defeasible rules $\Lambda$, is also an extension of the defeasible theory $\langle \Sigma \cup \{\varphi\}, D, \succ \rangle$ according to the property \emph{Cut}. Therefore, $\Lambda \not\subseteq \Omega$ and $\Omega \not\subseteq \Lambda$ according to Proposition \ref{unique}.

Consider an extension $\mathcal{E}'$ of the defeasible theory $\langle \Sigma, D, \succ \rangle$ determined by the defeasible rules $\Lambda$. Since $\varphi \in B$, there is an argument $A_\varphi$ generated by $\langle \Sigma, D, \succ \rangle$ such that $\tilde{A}_\varphi \cap \Lambda = \emptyset$. 

Every argument determined by the defeasible theory $\langle \Sigma, D, \succ \rangle$ is also an argument of the defeasible theory $\langle \Sigma \cup \{\varphi\}, D, \succ \rangle$. Moreover, every argument $A_\psi$ determined by the defeasible theory $\langle \Sigma \cup \{\varphi\}, D, \succ \rangle$ is either an argument of the defeasible theory $\langle \Sigma, D, \succ \rangle$, or contains $\{ \langle \emptyset, \varphi \rangle \}$ as a sub-argument. If we replace every sub-argument $\{ \langle \emptyset, \varphi \rangle \}$ in $A_\psi$ by $A_\varphi$, denoted by $A^*_\psi$, then we get an argument of the defeasible theory $\langle \Sigma, D, \succ \rangle$.

Consider the above mentioned extension $\mathcal{E}$ of $\langle \Sigma \cup \{\varphi\}, D, \succ \rangle$ determined by the defeated rules $\Omega$. Clearly, $\tilde{A}_\varphi \cap \Omega \not= \emptyset$ otherwise $\mathcal{E}$ would als be an extension of $\langle \Sigma, D, \succ \rangle$. 
Therefore, there is a defeasible rule $\alpha \leadsto \beta \in (\tilde{A}_\varphi \cap \Omega)$ and a corresponding argument $A_{\neg(\alpha \leadsto \beta)}$ of the defeasible theory $\langle \Sigma \cup \{\varphi\}, D, \succ \rangle$. In no extension $\mathcal{E}'$, $A^*_{\neg(\alpha \leadsto \beta)}$ is a valid argument. This is only possible if the validity of $A_{\neg(\alpha \leadsto \beta)}$ given $\Omega$ depends, directly or indirectly through arguments defeating other arguments, on $\varphi$. So, $A_{\neg(\alpha \leadsto \beta)}$ depends, directly or indirectly, on an argument that has $\{ \langle \emptyset, \varphi \rangle \}$ as a sub-argument. Hence, if we replace all arguments $A$ on which $A_{\neg(\alpha \leadsto \beta)}$ depends by $A^*$, then $A^*_{\neg(\alpha \leadsto \beta)}$ is part of an odd loop. This contradicts the condition of the theorem.
\end{proof}

\begin{theorem}
The defeasible theory $\langle \Sigma, D, \succ \rangle$ satisfies
{\em Exclusive Or}: \\
if $\Sigma \cup \{\varphi \wedge \neg\psi\} \nm \eta$, 
$\Sigma \cup \{\neg\varphi \wedge \psi\} \nm \eta$, then $\Sigma
\cup \{\varphi \xor \psi \} \nm \eta$;
\end{theorem}
\begin{proof}
Let $r_1 = \varphi \vee \psi \leadsto \varphi \wedge \neg \psi $,
$r_2 = \varphi \vee \psi \leadsto \varphi \wedge \psi$ and $r_3 =
\varphi \vee \psi \leadsto \neg \varphi \wedge \psi$.

To proof the theorem, we must prove that for every extension $\cal
E$ of the defeasible theory $\langle \Sigma \cup \{ \varphi \vee
\psi \}, D, \succ \rangle$, either that ${\cal E}$ is an extension
of $\langle \Sigma \cup \{\varphi \wedge \neg\psi\}, D, \succ \rangle$ or that
${\cal E}$ is an extension of $\langle \Sigma \cup \{\neg\varphi \wedge \psi\}, D,
\succ \rangle$. Since for every extension ${\cal E}'$ of
$\langle \Sigma \cup \{\varphi \wedge \neg\psi\}, D, \succ \rangle$ and of
$\langle \Sigma \cup \{\neg\varphi \wedge \psi\}, D, \succ \rangle$, $\eta \in {\cal
E}'$ holds, and since $ B = \bigcap_{i} {\cal E}_i$, $\Sigma
\cup \{\varphi \vee \psi\} \nm \eta$.

Let ${\cal E}$ be an extension of $\langle \Sigma \cup \{\varphi
\vee \psi\}, D, \succ \rangle$. Then because of the set of
hypotheses $H$, $\varphi \wedge \neg \psi \in {\cal E}$ or $\neg
\varphi \wedge \psi \in {\cal E}$. Notice that for no $\cal E$,
$\varphi \wedge \psi \in {\cal E}$ unless $\Sigma$ is
inconsistent.

Suppose that $\varphi \wedge \neg \psi \in {\cal E}$. Then $\Omega
= \{ \alpha \leadsto \beta \mid \neg(\alpha \leadsto \beta) \in
{\cal E} \}$. To prove that ${\cal E}$ is an extension of $\langle
\Sigma \cup \{\varphi \wedge \neg\psi\}, D, \succ \rangle$, we have to prove that
$\alpha \in {\cal E}$ if and only if there is an argument
$A^*_\alpha$ such that $A^*_\alpha \cap \Omega = \emptyset$ given
$\langle \Sigma \cup \{\varphi \wedge \neg\psi\}, D, \succ \rangle$.

Let $\alpha \in {\cal E}$. Then there is an $A_\alpha$ given
$\langle \Sigma \cup \{ \varphi \vee \psi \}, D, \succ \rangle$
with $\tilde{A}_\alpha \cap \Omega = \emptyset$. Therefore, we can
construct an $A^*_\alpha$ given $\langle \Sigma \cup \{\varphi \},
D, \succ \rangle$ such that $A^*_\alpha \cap \Omega = \emptyset$ by
first replacing each sub-argument $\{ \langle \{ \langle \emptyset,
\varphi \vee \psi \rangle \}, r_1 \rangle \}$ in $A_\alpha$ by $\{
\langle \emptyset, \varphi \rangle \}$ and subsequently by replacing each
remaining sub-argument $\{ \langle \emptyset, \varphi \vee \psi
\rangle \}$ also by $\{ \langle \emptyset, \varphi \rangle \}$.
Hence, there is an argument $A^*_\alpha$ given $\langle \Sigma \cup
\{\varphi \wedge \neg\psi\}, D, \succ \rangle$ with 
$\tilde{A}^*_\alpha \cap \Omega = \emptyset$.

Let $A^*_\alpha$ be an argument given $\langle \Sigma \cup
\{\varphi \wedge \neg\psi\}, D, \succ \rangle$ with $\tilde{A}^*_\alpha \cap \Omega
= \emptyset$. Then we can construct an $A_\alpha$ given $\langle
\Sigma \cup \{\varphi \vee \psi\}, D, \succ \rangle$ by replacing
each sub-argument $\{ \langle \emptyset, \varphi \rangle \}$ in
$A^*_\alpha$ by $\{ \langle \{ \langle \emptyset, \varphi \vee \psi
\rangle \}, r_1 \rangle \}$. To make sure that $A_\alpha \cap
\Omega = \emptyset$, we must make sure that $r_1$ is not defeated.
If $r_1$ is defeated, there must be a valid argument for $\neg
(\varphi \wedge \neg \psi)$. Since $\varphi \wedge \neg \psi \in
{\cal E}$, there is no such argument.

\vspace{2mm}

In case $\neg \varphi \wedge \psi \in {\cal E}$, the proof is
similar to the one given above.
\end{proof}

\section*{Appendix B}
Associate with each node and with each justification of the JTMS a
counter. Initially, set the counter of a node equal to the number
of incoming justifications and the counter of each justification
equal to the number of the of out-nodes of the justification.
Determine all the nodes that have a justification with an empty
set of out-nodes. Label these nodes {\sc in}, and place them on
the in-list. Next execute {\em propagate}.

\begin{tabbing}
    mmm\=mmm\=mmm\=mmm\=mmm\=mmm\=mmm\=mmm\= \kill
    {\em propagate}: \\
    \> {\bf for} each node on the in-list \\
    \> and for each out-going justification {\bf do} \\
    \>\> decrement the counter of its consequent node;  \\
    \>\> remove the justification;  \\
    \>\> {\bf if} the counter of the node is equal to 0 {\bf then} \\
    \>\>\> label the node OUT;  \\
    \>\>\> place the node on the out list;  \\
    \>\> {\bf end} \\
    \> {\bf end} \\
    \> delete the in-list;  \\
    \> {\bf for} each node on the out-list \\
    \> and for each out-going justification {\bf do} \\
    \>\> decrement the counter of the justification;  \\
    \>\> {\bf if} the counter of the justification is equal to 0 {\bf then} \\
    \>\>\> label its consequent node IN;  \\
    \>\>\> place its consequent node on the in-list;  \\
    \>\> {\bf end} \\
    \> {\bf end};  \\
    \> delete the out-list;  \\
    \> {\bf if} the in-list is not empty {\bf then} \\
    \>\> {\bf repeat} {\em propagate};  \\
    \> {\bf end} \\
    {\bf end}
\end{tabbing}

The above described procedure need not result in a complete
labeling of the JTMS. When this is the case, more than one labeling exist. To create a complete labeling, we must
choose one of the unlabeled nodes, a node that is not labeled {\sc
in}, {\sc out} or {\sc undetermined}, and label in {\sc in} or
{\sc out}. If we label the node {\sc in}, we place the node on the
in-list, if we label it {\sc out,} we place it on the out-list.
Subsequently, we must execute the procedure {\em propagate}.

We repeat the selection of a node, giving it a label and
propagating the consequences, till all nodes are labeled. By
backtracking on the choices that are made, we determine every
labeling of the JTMS.

\section*{\rm Acknowledgment}
I thank the reviewers and Cees Witteveen for their comments which
helped me to improve the paper.

%
%
%
%
%

\end{document}